%% file: v6.tex
\documentclass{article}

\usepackage[final, main, nonatbib]{neurips_2026}

\usepackage[utf8]{inputenc} 
\usepackage[T1]{fontenc}    
\usepackage{hyperref}       
\usepackage{url}            
\usepackage{booktabs}       
\usepackage{amsfonts}       
\usepackage{nicefrac}       
\usepackage{microtype}      
\usepackage{xcolor}         
\usepackage{comment}
\usepackage[normalem]{ulem} 

\usepackage{amsmath,amssymb,amsthm}   
\usepackage{enumitem}                 
\usepackage{algorithm}                
\usepackage{algpseudocode}            

\newcommand{\R}{\mathbb{R}}
\newcommand{\E}{\mathbb{E}}

\DeclareMathOperator{\softmax}{softmax}
\DeclareMathOperator{\diag}{diag}
\DeclareMathOperator{\KL}{KL}
\DeclareMathOperator{\tr}{tr}

\newtheorem{theorem}{Theorem}
\newtheorem{proposition}{Proposition}
\newtheorem{remark}{Remark}
\newtheorem{lemma}{Lemma}

\usepackage{multirow}
\usepackage{graphicx}

\title{
Minimally Invasive Steering of Language Models
}

\author{%
  Taha Entesari, Jingyu Zhang, Daniel Khashabi, Mahyar Fazlyab\\
  Johns Hopkins University \\
  Corresponding author: \texttt{mahyarfazlyab@jhu.edu}
}

\begin{document}

\maketitle

\begin{abstract}
Test-time alignment adapts large language models (LLMs) to runtime objectives without updating model parameters. Pre-logit steering offers a lightweight approach by optimizing additive interventions to the hidden states of a frozen LLM, but excessive steering can distort output distributions and degrade generation quality. We propose Minimally Invasive Steering Vector Optimization (MISVO), which improves reward while controlling deviation from the reference policy. Rather than penalizing the Euclidean norm of steering vectors, we derive a quadratic regularizer from the local Kullback--Leibler (KL) geometry of the induced token distributions. The resulting Fisher-information penalty captures the distributional sensitivity of steering directions and admits an analytic gradient computed through matrix--vector products with the frozen language-model head. We further derive an exact decomposition of the sequence-level KL gradient and establish explicit, non-asymptotic bounds on the approximation errors, including the omitted trajectory-dependent score-function term. These results show that the frozen Fisher regularizer gradient approximates the full sequence-level KL gradient to first order in the steering magnitude. Across preference and code-generation tasks on four LLMs spanning approximately 1B--14B parameters, MISVO achieves the highest mean best-so-far reward in six of seven model--task settings, while maintaining diversity and coherence scores comparable to Best-of-$N$.
\end{abstract}

\section{Introduction}
\label{sec:intro}

Post-training methods, such as reinforcement learning from human
feedback~\cite{ouyang2022training}, align large language models (LLMs)
by updating their parameters. At deployment, however, the target
reward may be user-specific, context-dependent, time-varying, or
accessible only through a black-box evaluator. Such settings motivate
\textit{test-time alignment}, which adapts generation to a given
reward objective without modifying model
parameters~\cite{khanov2024args,mudgal2024controlled,kong2024aligning,
kanai2025test,zhang2025controllablesafetyalignment}.

A promising approach to test-time alignment is
\textit{activation steering}, which modifies model behavior through
additive interventions, known as steering vectors, in its hidden
representations~\cite{subramani2022extracting,turner2023steering}.
These interventions may be fixed in advance or optimized at inference
time for the current prompt and reward
objective~\cite{kong2024aligning,kanai2025test}.
We focus on \textit{pre-logit steering}, in which steering vectors
are added to the hidden states immediately before the
language-model (LM) head~\cite{kanai2025test}.
This parameterization enables reward-driven adaptation without
backpropagating through the transformer backbone.
However, unconstrained steering can drive hidden states away from
the model's naturally occurring representations
~\cite{mishra2026steeredllmactivationsnonsurjective}, potentially
distorting token distributions and degrading generation quality.
This phenomenon, known as \textit{oversteering}, can be exacerbated
by optimization against imperfect reward functions, which may favor
outputs that exploit the evaluator rather than genuinely improve
response quality~\cite{gao2022scaling,krakovna2020specification}.

To address this challenge, we propose \textit{minimally invasive
steering}, which seeks to improve reward while limiting the
distributional deviation from the reference policy.
The Kullback--Leibler (KL) divergence provides a natural measure
of this deviation~\cite{ouyang2022training,korbak2022rl,
go2023aligning}. Directly optimizing the sequence-level KL,
however, is challenging: steering changes not only the current
token distribution but also the distribution of subsequent
prefixes, introducing a trajectory-dependent score-function term
whose Monte Carlo estimation can have high variance.
A quadratic regularizer offers a computationally tractable
alternative, but the conventional isotropic penalty on the
squared Euclidean norm of steering vectors fails to account
for their distributional effects. Indeed, steering directions
of equal magnitude can induce substantially different changes
in token probabilities.

We instead derive a quadratic effort penalty from the
\textit{local KL geometry} of the token distribution.
The resulting cost matrix is the Fisher information induced
by the frozen LM head, assigning steering directions penalties
according to their effects on the output distribution rather
than their Euclidean magnitudes. The regularizer admits an
analytic gradient computed using matrix--vector products with
the LM head, without differentiating through the transformer
backbone or sampled trajectories. Beyond the token-level
approximation, we establish explicit bounds on the discrepancy
between the Fisher surrogate and the \textit{full sequence-level
KL gradient}, including the omitted trajectory-dependent term.
These results show that the frozen Fisher gradient agrees
with the full KL gradient to first order in the steering
magnitude for a fixed generation horizon.

Building on this analysis, we introduce \textit{Minimally Invasive
Steering Vector Optimization} (MISVO), a test-time algorithm
that combines score-function reward-gradient estimates with
an analytic Fisher-based regularizer gradient. MISVO keeps
all model parameters frozen and reuses a reference-policy
Fisher estimate throughout optimization. Our contributions
are summarized as follows:
\begin{itemize}[topsep=0pt,itemsep=0pt,leftmargin=*]

\item \textbf{Distribution-sensitive steering.}
We formulate pre-logit steering as a reward--effort trade-off
and derive a quadratic regularizer from the second-order
expansion of token-level KL divergence
(Proposition~\ref{prop:kl-quad}).
The resulting Fisher quadratic admits an exact variance
interpretation in terms of steering-induced logit perturbations.

\item \textbf{Sequence-level KL guarantees.}
We derive an exact decomposition of the sequence-level KL
gradient into an analytic Fisher component and a
trajectory-dependent suffix term
(Theorem~\ref{thm:seqkl-grad}).
We establish explicit, non-asymptotic bounds on the Fisher
approximation errors and the omitted suffix term, showing
that the frozen Fisher gradient approximates the full
sequence-level KL gradient with second-order error
(Theorem~\ref{thm:surrogate-equiv}).

\item \textbf{Algorithm and empirical evaluation.}
We develop MISVO, a reward-driven test-time steering
algorithm with an analytic Fisher regularizer gradient
(Algorithm~\ref{alg:misvo}).
We evaluate reward, generation quality, distributional
deviation, computational runtime, and memory consumption
on SHP and MBPP+ across four frozen LLMs spanning
approximately 1B--14B parameters.

\end{itemize}

\section{Problem Setup}
\label{sec:problem}

\begin{figure}[t!]
    \centering
    \includegraphics[width=0.95\linewidth]{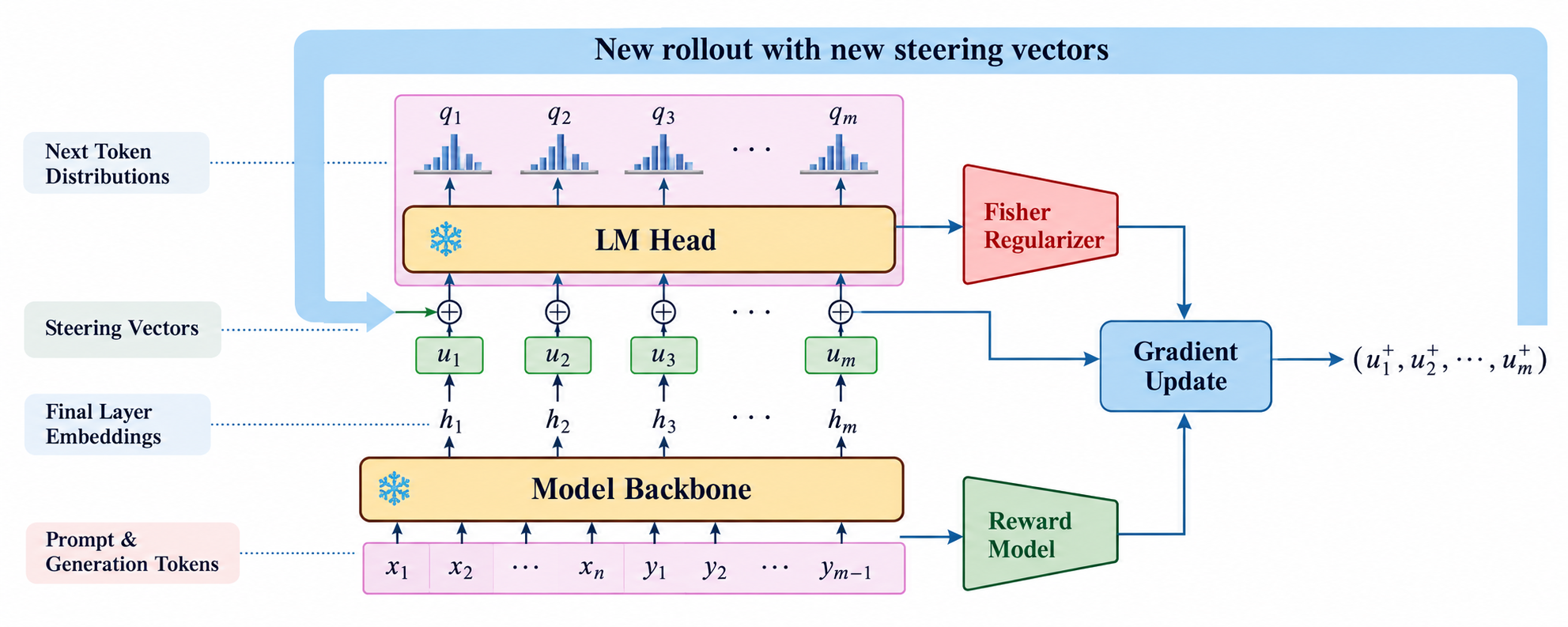}
    \caption{\textbf{Minimally invasive test-time steering.}
    \small The transformer backbone and LM head $(W,b)$ remain frozen.
    At each generation position $t$, MISVO adds a steering vector $u_t$
    to the pre-logit hidden state $h_t$, thereby modifying the next-token
    distribution. The steering vectors are optimized at inference time to
    increase reward while a Fisher-quadratic penalty controls the
    distributional deviation from the reference policy.}
    \label{fig:misvo}
\end{figure}

Fix a prompt $x$ and a generation horizon $T$. Let
$y=(y_1,\dots,y_T)$ denote a response over a vocabulary of size $V$, and let
\(
h_t = h_t(x,y_{<t}) \in \R^d
\)
be the frozen model's hidden state immediately before the language-model
(LM) head at position $t$. With frozen head parameters
$W\in\R^{V\times d}$ and $b\in\R^V$, the reference next-token policy and
the corresponding sequence distribution are
\begin{equation*}
\pi_{\mathrm{ref}}(\cdot\mid x,y_{<t})
=
\softmax(Wh_t+b),
\qquad
\pi_{\mathrm{ref}}(y\mid x)
=
\prod_{t=1}^T
\pi_{\mathrm{ref}}(y_t\mid x,y_{<t}).
\end{equation*}
An external evaluator assigns a terminal reward
$R(x,y)\in\R$, which may be non-differentiable or available only through
black-box evaluations.

\paragraph{Pre-logit steering.}
We introduce position-specific steering vectors
$u=(u_1,\dots,u_T)$, with $u_t\in\R^d$, and add $u_t$ directly to the
pre-logit hidden state. This induces the steered policy
\begin{equation*}
\pi_u(\cdot\mid x,y_{<t})
=
\softmax\!\big(W(h_t+u_t)+b\big),
\qquad
\pi_u(y\mid x)
=
\prod_{t=1}^T
\pi_u(y_t\mid x,y_{<t}).
\end{equation*}
At $u=0$, the steered policy recovers the reference policy,
$\pi_0=\pi_{\mathrm{ref}}$. The transformer backbone and LM
head remain fixed: steering modifies only the $d$-dimensional variables
$u_{1:T}$ at inference.

\paragraph{KL-regularized test-time alignment.}
A canonical objective for test-time alignment is~\cite{korbak2022rl,kong2024aligning}
\begin{equation}
\label{eq:ideal_alignment}
\max_{q(\cdot\mid x)}
\;
\E_{y\sim q}[R(x,y)]
-
\lambda\,
\KL\!\big(
q(\cdot\mid x)
\,\|\, 
\pi_{\mathrm{ref}}(\cdot\mid x)
\big),
\qquad
\lambda>0,
\end{equation}
where the optimization is over response distributions and $\lambda$
controls the reward--deviation trade-off. The optimizer is the
exponentially tilted distribution
\begin{equation}
\label{eq:ideal_tilt}
q^\star(y\mid x)
=
\frac{
\pi_{\mathrm{ref}}(y\mid x)
\exp\!\big(R(x,y)/\lambda\big)
}{
\sum_{y'}
\pi_{\mathrm{ref}}(y'\mid x)
\exp\!\big(R(x,y')/\lambda\big)
}.
\end{equation}
Equation~\eqref{eq:ideal_tilt} characterizes the ideal distributional
solution when $q$ is unrestricted. Pre-logit steering, however, can
realize only the restricted family
\(
\mathcal{P}_{\mathrm{steer}}
:=
\left\{
\pi_u(\cdot\mid x)
:
u\in\R^{d\times T}
\right\}.
\)
In general, $q^\star\notin\mathcal{P}_{\mathrm{steer}}$,
since the reward-tilted distribution need not be realizable
through position-specific pre-logit interventions.
We therefore seek a steered policy that improves reward
while controlling deviation from $\pi_{\mathrm{ref}}$.

\section{Proposed Method: Minimally Invasive Steering}
\label{sec:method}

Restricting the KL-regularized alignment objective
\eqref{eq:ideal_alignment} to the family of policies induced by
pre-logit steering yields
\begin{equation}
\label{eq:steering-kl-objective}
\max_{u\in\R^{d\times T}}
\quad
\E_{y\sim\pi_u(\cdot\mid x)}[R(x,y)]
-\lambda\,
\KL\!\big(
\pi_u(\cdot\mid x)
\,\|\,
\pi_{\text{ref}}(\cdot\mid x)
\big).
\end{equation}
Equivalently, this problem seeks the steered policy closest to the
ideal reward-tilted distribution $q^\star$ in reverse KL divergence: $\min_{u} \KL(\pi_u(\cdot\mid x)\,\|\,q^\star(\cdot\mid x))$. 
By the autoregressive chain rule, the sequence-level KL decomposes
into expected per-token divergences:
\begin{equation}
\label{eq:kl-effort}
\begin{aligned}
\mathcal{E}_{\mathrm{KL}}(u)
\!:=\!
\KL\!\big(
\pi_u(\cdot\mid x)
\,\|\,
\pi_{\text{ref}}(\cdot\mid x)
\big) 
=
\!\sum_{t=1}^T
\E_{y_{<t}\sim\pi_u}\!
\left[
\KL\!\big(
\pi_u(\cdot\mid x,y_{<t})
\,\|\,
\pi_{\text{ref}}(\cdot\mid x,y_{<t})
\big)
\right].
\end{aligned}
\end{equation}
A derivation is provided in Appendix~\ref{app:chain-rule}.
Although each per-token KL admits a closed-form expression, the
expectation in \eqref{eq:kl-effort} is taken over prefixes generated
by the steered policy. Consequently, differentiating the
sequence-level KL introduces a trajectory-dependent score-function
term, whose Monte Carlo estimation can incur substantial variance.
This motivates replacing the KL penalty with a tractable surrogate
that retains its local distributional geometry.

\paragraph{Quadratic regularization.}
To obtain a tractable alternative to the sequence-level KL penalty, we consider
the \textit{quadratically} regularized steering objective
\begin{equation}
\label{eq:objective}
\begin{aligned}
\max_{u\in\R^{d\times T}}\quad
J(u)
&:=
\E_{y\sim\pi_u(\cdot\mid x)}[R(x,y)]
-\lambda\,\mathcal{E}_{\Sigma}(u) \quad \text{where} \quad
\mathcal{E}_{\Sigma}(u)
:=
\frac{1}{2}\sum_{t=1}^T u_t^\top\Sigma_tu_t,
\end{aligned}
\end{equation}
where $\Sigma_t\in\R^{d\times d}$ are symmetric positive semidefinite
cost matrices that determine the penalty assigned to different steering
directions. Treating these matrices as fixed during each gradient update,
the regularizer admits the analytic gradient
$\nabla_{u_t}\mathcal{E}_{\Sigma}(u)=\Sigma_tu_t$,
avoiding the trajectory-dependent score-function estimation associated
with the sequence-level KL.

The key question is how to choose $\Sigma_t$ so that the quadratic
penalty captures the distributional deviation induced by steering rather
than the Euclidean magnitude of the intervention. We first present the
resulting optimization algorithm and then derive a principled choice
of $\Sigma_t$ from the local Fisher geometry of the token distribution
in Section~\ref{sec:fisherCostMatrices}.

\subsection{Algorithm}
\label{sec:algorithm}

We optimize \eqref{eq:objective} using stochastic gradient ascent.
At each iteration, we draw $K>1$ independent responses
$y^{(i)}\sim\pi_u(\cdot\mid x)$ and evaluate their rewards
$R_i:=R(x,y^{(i)})$. For each position $t$, the ascent direction
combines a REINFORCE estimate of the reward gradient with the
analytic gradient of the quadratic regularizer:
\begin{equation}
\label{eq:update}
\begin{aligned}
g_t^R
&:=
\frac{1}{K}\sum_{i=1}^K
(R_i-\bar R_i)\,
W^\top\big(e_{y_t^{(i)}}-p_t^{(i)}\big),\\
d_t
&:=g_t^R-\lambda\,\Sigma_tu_t,
\qquad
u_t^+:=u_t+\eta d_t,
\end{aligned}
\end{equation}
where $p_t^{(i)}:=\pi_u(\cdot\mid x,y_{<t}^{(i)})$,
$e_{y_t^{(i)}}\in\R^V$ is the one-hot vector corresponding to
the sampled token, and $\eta>0$ is the step size.
We use the leave-one-out reward baseline
$\bar R_i:=(K-1)^{-1}\sum_{j\ne i}R_j$,
which reduces variance without introducing bias into the
reward-gradient estimator~\cite{ahmadian2024back}.
All positions are updated simultaneously using rollouts and
token probabilities computed under the same policy iterate.
Algorithm~\ref{alg:misvo} summarizes MISVO, which returns
the final steered policy together with the highest-reward
response encountered during optimization.

\begin{algorithm}[t!]
\caption{Minimally Invasive Steering Vector Optimization (MISVO)}
\label{alg:misvo}
\begin{algorithmic}[1]
\State \textbf{Input:} frozen model $\pi_{\text{ref}}$ with LM head $(W,b)$, prompt $x$, reward $R$, horizon $T$, cost matrices $\{\Sigma_t\}_{t=1}^T$, regularization weight $\lambda$, learning rate $\eta$, rollouts per step $K>1$, iterations $N$
\State Initialize $u_t\gets 0$ for $t=1,\dots,T$
\For{$n=1,\dots,N$}
    \State Sample $y^{(i,n)}\sim\pi_u(\cdot\mid x)$ independently for $i=1,\dots,K$
    \State Store $p_t^{(i)}\gets\pi_u(\cdot\mid x,y_{<t}^{(i,n)})$ for all $i,t$
    \State Compute $R_i\gets R(x,y^{(i,n)})$ and
    $\bar R_i\gets(K-1)^{-1}\sum_{j\ne i}R_j$ for all $i$
    \For{$t=1,\dots,T$}
        \State $g_t^R\gets \frac{1}{K}\sum_{i=1}^K
        (R_i-\bar R_i)W^\top(e_{y_t^{(i,n)}}-p_t^{(i)})$
        \State $u_t^+\gets u_t+\eta(g_t^R-\lambda\Sigma_tu_t)$
    \EndFor
    \State $u\gets u^+$
\EndFor
\State \textbf{Output:} policy $\pi_u$ and a response attaining
$\max_{i,n}R(x,y^{(i,n)})$
\end{algorithmic}
\end{algorithm}

\section{Fisher Cost Matrices}
\label{sec:fisherCostMatrices}

In this section, we derive the cost matrices $\Sigma_t$ in
\eqref{eq:objective} from the local geometry of KL divergence.
We then analyze the relationship between the resulting quadratic
regularizer and the sequence-level KL gradient.

\subsection{From token-level KL to a quadratic penalty}
\label{sec:kl-fisher}

The choice $\Sigma_t=I$ yields the isotropic penalty
$\mathcal{E}_{\Sigma}(u)=\tfrac12\sum_{t=1}^T\|u_t\|_2^2$,
which penalizes the magnitude of steering interventions without
accounting for their effect on the token distribution.
To capture this distributional sensitivity, we exploit the
classical information-geometric connection between KL divergence
and Fisher information: locally, the Fisher information matrix
characterizes the second-order variation of KL
divergence~\cite{amari1998natural,amari2000methods}.
We specialize this relation to pre-logit steering to construct
a quadratic penalty that reflects the distributional effect
of an intervention.

Fix a position $t$ and prefix $y_{<t}$. Let
$z=Wh_t+b$ denote the reference logits,
$\delta=Wu_t$ the steering-induced logit perturbation, and
$\pi_u^t=\pi_u(\cdot\mid x,y_{<t})$, with
$\pi_{\text{ref}}^t$ defined analogously.
Using the log-partition function
$A(z)=\log\sum_{i=1}^V e^{z_i}$,
the token-level KL admits the exact expression
\begin{equation}
\label{eq:per-token-kl-exact}
\KL(\pi_u^t\|\pi_{\text{ref}}^t)
=
A(z)-A(z+\delta)+\langle\pi_u^t,\delta\rangle.
\end{equation}
The Hessian of $A$ is the softmax covariance matrix,
\begin{equation}
\label{eq:softmax-jacobian}
\nabla^2 A(z)
=
\diag(\pi_{\text{ref}}^t)
-\pi_{\text{ref}}^t(\pi_{\text{ref}}^t)^\top
=:C_{\pi_{\text{ref}}^t}.
\end{equation}
Expanding \eqref{eq:per-token-kl-exact} around $u_t=0$ yields the following result, proved in Appendix~\ref{app:kl-quad-proof}.

\begin{proposition}[Fisher-quadratic approximation of token-level KL]
\label{prop:kl-quad}
For any fixed prefix $y_{<t}$, as $u_t\to0$,
\begin{equation}
\label{eq:kl-quadratic}
\KL(\pi_u^t\|\pi_{\text{ref}}^t)
=
\frac{1}{2}u_t^\top F_t(0)u_t
+O(\|u_t\|_2^3),
\end{equation}
where
\begin{equation}
\label{eq:fisher-u}
F_t(u_t)
:=
W^\top
\bigl(
\diag(\pi_u^t)-\pi_u^t(\pi_u^t)^\top
\bigr)W
\end{equation}
is the Fisher information matrix of the token distribution
with respect to the steering vector $u_t$.
\end{proposition}
Since $F_t(0)$ depends on the prefix $y_{<t}$,
we construct a position-specific cost matrix by averaging
over prefixes generated by the reference policy:
\begin{equation}
\label{eq:fisher-effort}
\begin{aligned}
\bar F_t
&:=
\E_{y_{<t}\sim\pi_{\text{ref}}}[F_t(0)], \qquad 
\mathcal{E}_{\bar F}(u)
:=
\frac{1}{2}\sum_{t=1}^T u_t^\top\bar F_tu_t.
\end{aligned}
\end{equation}
MISVO uses $\Sigma_t=\bar F_t$ in \eqref{eq:objective}.
Because $\bar F_t$ is computed under the reference policy,
it can be estimated once and held fixed throughout steering
optimization, yielding the analytic regularizer gradient
$\nabla_{u_t}\mathcal{E}_{\bar F}(u)=\bar F_tu_t$.

Relative to the sequence-level KL in \eqref{eq:kl-effort},
this construction introduces two approximations:
(i) replacing each token-level KL by its second-order
Fisher-quadratic expansion, and
(ii) evaluating the expected Fisher matrix under reference-policy
prefixes rather than steered-policy prefixes.
We analyze these approximations in the following subsections.

\begin{remark}[Variance interpretation]
\label{rem:kl-properties}
\normalfont
The Fisher quadratic admits the exact representation
\begin{equation}
\label{eq:fisher-variance}
u_t^\top F_t(0)u_t
=
\operatorname{Var}_{i\sim\pi_{\text{ref}}^t}
\bigl[(Wu_t)_i\bigr].
\end{equation}
Hence, the leading term in the token-level KL expansion
is one half of the variance of the steering-induced logit
perturbation under the reference distribution.
Unlike an isotropic penalty, the Fisher quadratic accounts
for how an intervention changes the relative logits,
weighted by their reference probabilities.
In particular, directions satisfying
$Wu_t=c\mathbf{1}_V$ for some $c\in\R$ incur zero
cost, consistent with the invariance of softmax to
constant logit shifts.
\end{remark}

The approximation in Proposition~\ref{prop:kl-quad}
is local: neither the token-level KL nor its Fisher-quadratic
approximation globally upper-bounds the other.
Nevertheless, both admit the common upper bound
\begin{equation*}
\max\left\{
\KL(\pi_u^t\|\pi_{\text{ref}}^t),
\frac12 u_t^\top F_t(0)u_t
\right\}
\leq
\frac12\|Wu_t\|_2^2
\leq
\frac12\sigma_{\max}(W)^2\|u_t\|_2^2.
\end{equation*}
as established in Appendix~\ref{app:kl-bound}.

\subsection{Sequence-level KL gradient}
\label{sec:seqkl-grad}

We next characterize the relationship between the proposed quadratic
regularizer and the gradient of the sequence-level KL divergence.
Let $K(u):=\mathcal E_{\mathrm{KL}}(u)$ and define the suffix
log-ratio
\begin{equation*}
\ell_{>t}(y;u)
:=
\sum_{\tau=t+1}^T
\log
\frac{\pi_u(y_\tau\mid x,y_{<\tau})}
{\pi_{\text{ref}}(y_\tau\mid x,y_{<\tau})}.
\end{equation*}

\begin{theorem}[Sequence-level KL gradient decomposition]
\label{thm:seqkl-grad}
For every position $t$, the sequence-level KL gradient admits
the exact decomposition
\begin{equation}
\label{eq:seqkl-grad}
\begin{aligned}
\nabla_{u_t}K(u)
&=g_t^{\mathrm{an}}(u)+g_t^{\mathrm{suffix}}(u),\\
g_t^{\mathrm{an}}(u)
&:=
\E_{y_{<t}\sim\pi_u}
\bigl[F_t(u_t)\bigr]u_t,\\
g_t^{\mathrm{suffix}}(u)
&:=
\E_{y\sim\pi_u}\!\left[
\nabla_{u_t}\log\pi_u(y_t\mid x,y_{<t})
\,\ell_{>t}(y;u)
\right].
\end{aligned}
\end{equation}
\end{theorem}

The proof is provided in Appendix~\ref{appendix:seqkl-grad}.
The decomposition separates the direct effect of $u_t$ on the
current token distribution from its indirect effect on future
token distributions through the generated prefix. At a fixed
prefix, the gradient of the token-level KL is exactly
$F_t(u_t)u_t$. Averaging over prefixes generated by the steered
policy yields the analytic component $g_t^{\mathrm{an}}(u)$.
The remaining term, $g_t^{\mathrm{suffix}}(u)$, accounts for how
changes in the current token distribution alter the expected
KL contributions at subsequent positions.

Comparing \eqref{eq:seqkl-grad} with the proposed regularizer
gradient $\bar F_tu_t$ identifies two sources of discrepancy:
(i) approximating the iterate-dependent Fisher matrix
$\E_{y_{<t}\sim\pi_u}[F_t(u_t)]$ by its reference-policy
counterpart $\bar F_t$, and
(ii) omitting the suffix score-function term.
In the next subsection, we quantify the first discrepancy
and establish the first-order equivalence of several
Fisher-based surrogates to the analytic component of the
sequence-level KL gradient.

\subsection{Accuracy of the Fisher surrogates}
\label{sec:approx-cost}

To quantify the accuracy of the proposed regularizer, consider
three Fisher matrices that differ in the distribution used to
generate prefixes and evaluate token probabilities:
\begin{align}
\tilde F_t^u
&:=\E_{y_{<t}\sim\pi_u}[F_t(u_t)],
\label{eq:surr-onon}\\
\bar F_t^u
&:=\E_{y_{<t}\sim\pi_u}[F_t(0)],
\label{eq:surr-onoff}\\
\bar F_t
&:=\E_{y_{<t}\sim\pi_{\text{ref}}}[F_t(0)].
\label{eq:surr-offoff}
\end{align}
The first recovers the exact analytic component of the
sequence-level KL gradient,
$g_t^{\mathrm{an}}(u)=\tilde F_t^u u_t$.
The second evaluates the Fisher matrix at zero steering while
retaining prefixes generated by the steered policy.
The third additionally replaces the steered prefix distribution
with the reference distribution, yielding the fixed cost
matrix used by MISVO.

\begin{theorem}[Accuracy of the Fisher surrogates]
\label{thm:surrogate-equiv}
Let $B_W:=\sigma_{\max}(W)$ and define
$C_1:=3B_W^3$ and $C_2:=\sqrt{V}B_W^3$.
For every steering sequence $u$ and position $t$,
\begin{equation}
\label{eq:surrogate-bounds}
\begin{aligned}
\|\tilde F_t^u-\bar F_t^u\|_2
&\le C_1\|u_t\|_2, \quad \|\bar F_t^u-\bar F_t\|_2
\le C_2\sum_{\tau<t}\|u_\tau\|_2.
\end{aligned}
\end{equation}
Moreover, the suffix score-function term in
Theorem~\ref{thm:seqkl-grad} satisfies
\begin{equation}
\label{eq:suffix-second-order-bound}
\|g_t^{\mathrm{suffix}}(u)\|_2
\le
B_W^3\sum_{\tau>t}\|u_\tau\|_2^2.
\end{equation}
Consequently, the full sequence-level KL gradient obeys
\begin{equation}
\label{eq:full-gradient-error}
\begin{aligned}
\|\nabla_{u_t}K(u)-\bar F_tu_t\|_2
\le
C_1\|u_t\|_2^2
\!+\!C_2\|u_t\|_2\sum_{\tau<t}\|u_\tau\|_2
+B_W^3\sum_{\tau>t}\|u_\tau\|_2^2.
\end{aligned}
\end{equation}
In particular, for a fixed generation horizon $T$,
\begin{equation}
\label{eq:surrogate-ladder}
\nabla_{u_t}K(u)
=
\bar F_tu_t+O(\|u\|_2^2),
\qquad u\to0,
\end{equation}
where $\|u\|_2$ denotes the Euclidean norm of the
stacked steering vectors.
\end{theorem}

The proof is provided in Appendix~\ref{app:proofOfSurrogateEquiv}.
The first two bounds quantify the errors introduced by
evaluating the Fisher matrix at zero steering and replacing
the steered prefix distribution with the reference distribution.
The third controls the omitted suffix score-function term.
Together, these results establish that the frozen-Fisher
regularizer gradient agrees with the full sequence-level
KL gradient to first order. This equivalence is local in the steering magnitude and holds
for a fixed generation horizon. It does not guarantee accuracy
for large interventions or control the accumulation of
approximation errors over multiple optimization iterations.

\paragraph{Estimating the frozen Fisher matrices.}
Among the three surrogates, only $\bar F_t$ is independent
of the current steering iterate. Since MISVO initializes
$u=0$, its initial $K$ rollouts are sampled from
$\pi_{\text{ref}}$ and can be used to estimate
\begin{equation}
\label{eq:empirical-fisher}
\widehat F_t
:=
\frac{1}{K}\sum_{i=1}^K
F_t\!\left(0;x,y_{<t}^{(i,1)}\right),
\end{equation}
where we explicitly display the prefix dependence of
$F_t$. This estimator satisfies
$\E[\widehat F_t]=\bar F_t$.
We reuse the resulting estimate throughout optimization,
avoiding additional generations for Fisher estimation.

\subsection{Computational overhead}
\label{sec:complexity}

Over $N$ optimization iterations with $K$ rollouts per
iteration, MISVO generates $NK$ responses and evaluates
their rewards, matching the candidate budget of Best-of-$N$
(BoN). Beyond generation and reward evaluation, MISVO
incurs additional computation for reward-gradient estimation,
Fisher regularization, and steering updates. Computing the
REINFORCE gradient in \eqref{eq:update} requires projecting
vocabulary-dimensional score vectors through $W^\top$,
with total arithmetic cost $O(NKTVd)$ under dense
matrix--vector multiplication.

\paragraph{Explicit Fisher matrices.}
For a token distribution $p$, explicitly constructing
\(
F_t=W^\top\bigl(\diag(p)-pp^\top\bigr)W
\)
costs $O(Vd^2)$. Estimating the frozen matrices
$\widehat F_t$ from $S$ reference rollouts therefore
requires $O(STVd^2)$ preprocessing operations and
$O(Td^2)$ storage for the resulting matrices.
Once constructed, the regularizer gradients
$\widehat F_tu_t$ require $O(Td^2)$ operations per
iteration, giving a total regularization cost of
$O(STVd^2+NTd^2)$.
By contrast, explicitly reconstructing the on-policy
Fisher matrices at every iteration would require
$O(NKTVd^2)$ operations.

\paragraph{Matrix-free Fisher products.}
Explicit matrix construction can be avoided by exploiting
the covariance structure of the softmax. For
$v:=Wu_t$ and a token distribution $p$,
\(F_tu_t
=
W^\top\!\left(
p\odot v-p(p^\top v)
\right)\), 
where $\odot$ denotes elementwise multiplication.
Each Fisher--vector product requires $O(Vd)$
operations and avoids storing a dense $d\times d$
matrix.

For the frozen Fisher surrogate, we retain the reference
hidden states from $S$ initial rollouts and compute
$\widehat F_tu_t$ by averaging the corresponding
matrix-free products. This requires $O(NTSVd)$
arithmetic operations over the optimization iterations,
in addition to the cost of reconstructing reference
token probabilities when they are not cached.
Retaining the reference hidden states requires
$O(TSd)$ storage, while temporary vocabulary-sized
tensors introduce additional working-memory overhead.
Our implementation uses $S=K$ reference rollouts
and the matrix-free formulation.
Measured wall-clock runtime and peak GPU memory
are reported in Appendix~\ref{app:runtime}.

\section{Experiments}
\label{sec:experiments}

We evaluate MISVO on preference-based generation and code generation across four frozen models with approximately 1B--14B parameters. We compare reward at matched generation budgets, assess diversity and coherence, and measure distributional deviation on reference-policy prefixes. The code can be found in \url{https://github.com/o4lc/MISVO}.

\subsection{Setup}
\label{sec:exp-setup}

\paragraph{Models.}
We use four open-weight instruction-tuned base models with different architectures, hidden sizes, and parameter counts: LFM2.5-1.2B-Instruct~\cite{liquidai2025lfm2}, Gemma3-4B-IT~\cite{gemmateam2025gemma3technicalreport}, Llama-3-8B-Instruct~\cite{llama3modelcard}, and Phi-4 (14B)~\cite{abdin2024phi4technicalreport}.
All four models are kept frozen throughout: only the per-token steering vectors $u_{1:T}$ are optimized, and the LM head $(W, b)$ is reused both for the forward pass and for the matrix-free Fisher--vector products in~\eqref{eq:fisher-u}.

\paragraph{Tasks.}
We consider two complementary alignment settings.
\textbf{(a)~SHP}~\cite{pmlr-v162-ethayarajh22a}: open-ended preference-style generation. We sample $500$ prompts uniformly at random from the Stanford Human Preferences test set and score each generation with Skywork-Reward-V2-Qwen3-0.6B~\cite{liu2025skywork}, a publicly released $0.6$B-parameter preference reward model. We evaluate all four base models on this task.
\textbf{(b)~MBPP+}~\cite{evalplus}: program synthesis with an executable verifier. We use the first set of $120$ problems and reward each generation by the fraction of held-out unit tests it passes (in $[0,1]$); this provides a black-box, non-differentiable reward. We evaluate the three smaller models (LFM2.5-1.2B, Gemma3-4B, Llama-3-8B) on MBPP+. For both tasks, we cap generation at $512$ new tokens.

\paragraph{Best-so-far reporting.}
Our primary metric is the highest reward among all $KN$ generations, where $K$ is the number of rollouts per step and $N$ is the number of optimization steps. MISVO and AISP generate candidates iteratively; Best-of-$N$ samples all candidates from the reference policy. This metric evaluates the response selected after search. We also report within-step mean and maximum rewards to assess the evolving steered policy.

\paragraph{Baselines.}
\textbf{Best-of-$N$ (top-$p$)} draws $K \cdot N$ candidate completions from the reference policy $\pi_{\text{ref}}$ under the shared sampler and returns the highest-reward one; the budget is matched exactly to the total number of rollouts drawn by \textsc{MISVO} and AISP. 
\textbf{AISP}~\cite{kanai2025test} performs sampling-based optimal control in pre-logit space using the importance-weighted update. We set $\sigma^2 = 0.5$ and $\alpha=0.999$ for their method.
In line with the experimental results in \cite{kanai2025test}, experiments with \textit{RE-Control} \cite{kong2024aligning} yielded inferior results, and due to computational limitations, we did not include it in our comparisons.

\paragraph{Our method.}
MISVO uses the frozen-reference Fisher $\bar F_t$ in~\eqref{eq:fisher-effort}. We estimate it from the first-step rollouts, for which $u=0$, and reuse the estimate throughout optimization. These rollouts are included in the $KN$ budget.

\input{tables/shp_skywork}
\input{tables/mbpp}

\paragraph{Metrics.}
We average all metrics over prompts. \textbf{Reward} is the Skywork score on SHP and the unit-test pass fraction on MBPP+, reported for the best response found during search. \textbf{Diversity} is $\prod_{n\in\{2,3,4\}}(1-\mathrm{rep}_n)$, where $\mathrm{rep}_n$ is the fraction of repeated $n$-grams within a response~\cite{kanai2025test}. \textbf{Coherence} is the cosine similarity between SimCSE embeddings of the prompt and response~\cite{gao2021simcse}. The latter two metrics measure repetition and semantic relatedness; they do not establish distributional proximity, factual accuracy, or absence of reward exploitation.

\paragraph{Implementation details.}
For each (model, task) pair we use the same \textsc{MISVO} and AISP hyperparameters across all prompts in the split. On both datasets we use $K = 16$, $N = 16$ for LFM2.5-1.2B and Gemma3-4B, and $K = 32$, $N = 32$ for Llama-3-8B and Phi-4; the learning rate is $0.1$ and the regularization trade-off is $\lambda=1$.  
All runs use the Hugging Face \texttt{transformers} stack, fp16 weights for Llama-3 and bfloat16 for the others, and run on H200 GPUs. We report results as mean $\pm$ standard deviation across 3 seeds.

\subsection{Main results}
\label{sec:exp-main}

Across both tasks (Tables~\ref{tab:shp} and~\ref{tab:mbpp}), \textsc{MISVO} achieves the highest reported mean reward in six of the seven model--task settings. On SHP it improves over both BoN and AISP on three of the four models; the single exception is Llama-3-8B, where \textsc{MISVO} still improves over BoN but is overtaken by AISP. On MBPP+ \textsc{MISVO} leads on all three models, with the largest gain on the smallest. 

Diversity and coherence scores remain close to those of BoN. These diagnostics suggest that reward gains are not accompanied by substantial deterioration in the measured properties, but do not establish proximity to $\pi_{\mathrm{ref}}$. Section~\ref{sec:measured-kl} evaluates token-distribution deviation directly.

\paragraph{Reward trajectories.}
Figure~\ref{fig:lfmTrajectoryShp} reports the within-step mean reward, within-step maximum reward, and cumulative best reward for LFM2.5-1.2B on SHP. Within-step statistics use the $K$ rollouts from each iteration; the horizontal axis gives cumulative generations. Curves show means over three seeds, with shading indicating one standard deviation. Corresponding results for the other SHP models appear in Figures~\ref{fig:gemmaTrajectoryShp}, \ref{fig:llamaTrajectoryShp}, and~\ref{fig:phiTrajectoryShp}. MISVO's increasing within-step mean supports improvement in its sampling policy. AISP's declining mean at later iterations is consistent with oversteering, although this metric alone does not identify the cause.

\begin{figure}[ht]
    \centering
    \includegraphics[width=\linewidth,trim=0cm 0.99cm 0cm 0cm]{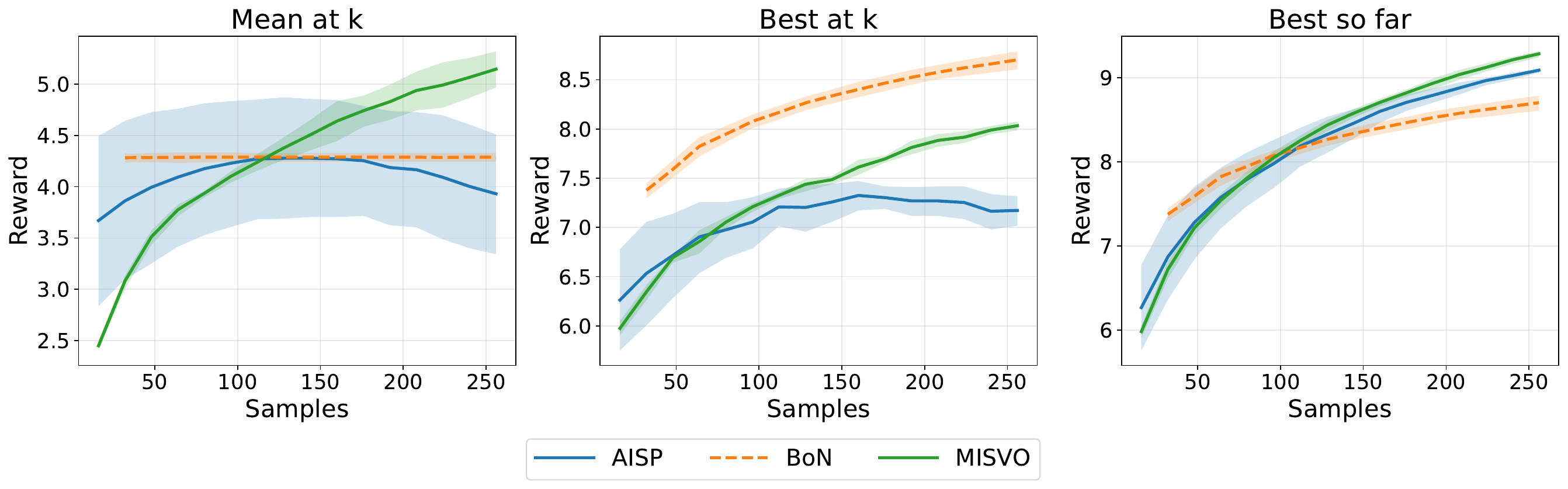}
    \caption{Reward trajectory for LFM2.5-1.2B/SHP as a function of the cumulative number of sampled completions. \textbf{Left:} mean reward within each step. \textbf{Middle:} maximum reward within each step. \textbf{Right:} cumulative maximum reward.}
    \label{fig:lfmTrajectoryShp}
\end{figure}

\subsection{Direct measurement of policy deviation}
\label{sec:measured-kl}

To assess distributional deviation beyond generation-level diagnostics, we measure token-level KL on reference-policy prefixes. Given $K$ reference rollouts, define

\begin{equation}
\label{eq:kl-estimator}
  \widehat{\text{KL}}_\text{ref}(u)
  \;=\;
  \frac{1}{K}\sum_{i=1}^{K}\sum_{t}
  \mathrm{KL}\!\left(
    \pi_{u}(\cdot \mid x, y^{(i)}_{<t})
    \;\big\|\;
    \pi_{\mathrm{ref}}(\cdot \mid x, y^{(i)}_{<t})
  \right),
  \quad 
  y^{(i)} \sim \pi_\text{ref}
\end{equation}
where each token-level KL is computed over the full vocabulary. This estimates the sum of expected token KLs under \emph{reference-policy} prefixes. It is distinct from the sequence-level KL in~\eqref{eq:kl-effort}, whose expectations use \emph{steered-policy} prefixes. The diagnostic measures deviation on reference contexts and does not capture all changes in prefix visitation induced by steering.

\input{tables/policyDeviation}

Table~\ref{tab:measured-kl} reports results over 100 sampled prompts. MISVO attains the highest reported reward and lower reference-prefix KL than AISP and unregularized steering. The remaining regularizers do not exceed BoN in reward in this comparison. The direct-KL variant also has lower reward and higher reference-prefix KL than MISVO. These results describe the evaluated configurations; they do not establish a global reward--KL frontier or identify the cause of the direct-KL variant's performance.

\section{Conclusion}
\label{sec:conclusion}

We introduced MISVO, which optimizes pre-logit interventions using a Fisher-quadratic penalty derived from local KL geometry. An exact sequence-level gradient decomposition and local error bounds establish first-order agreement between the frozen-Fisher gradient and the full KL gradient. Across seven model--task settings, MISVO achieves the highest reported mean reward in six, with diversity and coherence scores close to those of BoN.

\paragraph{Limitations.}
The approximation is local and does not guarantee small sequence-level KL after large interventions. The empirical deviation metric uses reference-policy prefixes and therefore does not directly measure sequence-level KL. On SHP, the same reward model is used for optimization and primary evaluation, limiting conclusions about independent preference quality. MISVO also incurs iterative optimization overhead and, in our implementation, higher peak memory than the baselines. Its integration into high-throughput inference systems and the effects of more memory-efficient Fisher products remain directions for future work.

\section*{Acknowledgment}
This work was supported in part by the JHU Provost Discovery Award (2025–2027). M.F. was supported in part by DARPA Grant No. D26AC50010, NSF Grant No. 2515978, and the U.S. Department of Energy, Office of Science, ASCR, through the SciDAC Institute “LEADS: LEarning-Accelerated Domain Science.” D.K. was supported in part by DARPA under Contract No. HR001125C0304 and ONR under Grant No. N0001424-1-2089. The views and conclusions expressed herein are those of the authors and do not necessarily reflect the views of DARPA. We acknowledge computational resources provided by the Johns Hopkins Data Science and AI Institute (DSAI) and thank Benjamin Van Durme and Gillian Hadfield for helpful discussions.

\clearpage

\bibliography{ref}

\clearpage

\appendix

\section{Related Work}
\label{sec:related_works}

\paragraph{Training-time alignment.}
RLHF optimizes model parameters using preference-derived rewards~\cite{christiano2017deep, ouyang2022training, stiennon2020learning}. DPO expresses a KL-regularized preference objective as a classification loss~\cite{rafailov2024direct}; other approaches modify the preference loss or feedback source~\cite{azar2024general, ethayarajh2024kto, meng2024simpo, bai2022constitutional}. These methods produce model weights for deployment. MISVO instead optimizes interventions for the reward available at inference.

\paragraph{Prompt-based control.}
Instructions, demonstrations, chain-of-thought prompting, and self-refinement guide generation through token inputs~\cite{brown2020language, dong2024survey, wei2022chain, wang2023selfconsistency, madaan2023self}. Their effectiveness depends on the model's response to the supplied context. MISVO directly modifies the pre-logit representation and explicitly regularizes the resulting distributional change.

\paragraph{Output-level test-time alignment.}
Best-of-$N$ selects the highest-reward completion among reference-policy samples~\cite{stiennon2020learning, beirami2024theoretical}. Guided decoding modifies token probabilities using reward models, value functions, or auxiliary scorers~\cite{khanov2024args, mudgal2024controlled, krause2021gedi, liu2021dexperts, yang2021fudge, xu2025genarm}. MISVO instead optimizes a sequence of hidden-space interventions using a terminal reward.

\paragraph{Representation-level test-time alignment.}
RE-Control uses a learned value function to guide hidden-state interventions~\cite{kong2024aligning}. AISP formulates pre-logit steering as sampling-based control and updates perturbations through importance-weighted reward evaluations~\cite{kanai2025test}. BRT-Align uses latent backward reachability and a learned safety value function to determine when to intervene~\cite{karnik2025brt}. Static activation-steering methods estimate directions from offline data~\cite{subramani2022extracting, turner2023steering, rimsky2024steering, li2023inferencetime, zou2023representation}; input-dependent variants also use optimal transport~\cite{abdullaev2026concept}. MISVO optimizes position-specific interventions against a runtime reward and regularizes them using the Fisher geometry of the output distribution. Its regularizer requires no auxiliary network, and its gradient agrees locally with the full sequence-level KL gradient to first order.

\newpage

\section{Theoretical Results}
\label{app:proofs}

\subsection{Preliminary lemmas}
\label{app:lemmas}

We first collect several standard properties of the softmax
function and categorical covariance matrices used throughout
the theoretical analysis. Recall that
\[
A(z):=\log\sum_{i=1}^V e^{z_i},
\qquad
C_p:=\diag(p)-pp^\top,
\]
where $p\in\R^V$ is a probability vector.

\begin{lemma}[Softmax log-partition identities]
\label{lem:logpart}
For every $z\in\R^V$,
\[
\nabla A(z)=\softmax(z),
\qquad
\nabla^2 A(z)=C_{\softmax(z)}.
\]
\end{lemma}

\begin{proof}
Differentiating the log-partition function gives
\[
\partial_i A(z)
=
\frac{e^{z_i}}{\sum_j e^{z_j}}
=
\softmax(z)_i.
\]
Writing $p=\softmax(z)$, differentiation once more yields
$\partial_j p_i=p_i(\mathbf{1}_{i=j}-p_j)$,
which gives $\nabla^2 A(z)=\diag(p)-pp^\top=C_p$.
\end{proof}

\begin{lemma}[Properties of categorical covariance matrices]
\label{lem:sigma}
For any probability vector $p\in\R^V$,
\[
C_p\succeq0,
\qquad
C_p\mathbf{1}=0,
\qquad
\|C_p\|_2\leq1,
\qquad
\tr(C_p)=1-\|p\|_2^2.
\]
\end{lemma}

\begin{proof}
Let $X$ be a random vector taking value $e_i$ with
probability $p_i$. Since
$C_p=\operatorname{Cov}(X)$, we have $C_p\succeq0$.
Moreover,
\[
C_p\mathbf{1}
=
p-p(p^\top\mathbf{1})=0.
\]
For any unit vector $v\in\R^V$,
\[
\begin{aligned}
v^\top C_pv
&=\operatorname{Var}(v^\top X)\\
&\leq\E[(v^\top X)^2]
=\sum_i p_iv_i^2\\
&\leq\|v\|_\infty^2\leq1.
\end{aligned}
\]
Thus, $\|C_p\|_2=\lambda_{\max}(C_p)\leq1$.
Finally,
\[
\tr(C_p)
=
\tr(\diag(p))-\tr(pp^\top)
=
1-\|p\|_2^2.
\]
\end{proof}

\begin{lemma}[Softmax Lipschitz continuity]
\label{lem:softmax-lip}
For any $z,z'\in\R^V$,
\[
\begin{aligned}
\|\softmax(z')-\softmax(z)\|_2
&\leq\|z'-z\|_2,\\
\|\softmax(z')-\softmax(z)\|_1
&\leq\sqrt V\,\|z'-z\|_2.
\end{aligned}
\]
\end{lemma}

\begin{proof}
Let $\delta=z'-z$ and define
$\gamma(s):=\softmax(z+s\delta)$ for $s\in[0,1]$.
By Lemma~\ref{lem:logpart},
$\gamma'(s)=C_{\gamma(s)}\delta$.
The fundamental theorem of calculus gives
\[
\softmax(z')-\softmax(z)
=
\int_0^1 C_{\gamma(s)}\delta\,ds.
\]
Using Lemma~\ref{lem:sigma},
\[
\|\softmax(z')-\softmax(z)\|_2
\leq
\int_0^1\|C_{\gamma(s)}\|_2\|\delta\|_2\,ds
\leq\|\delta\|_2.
\]
The second inequality follows from
$\|v\|_1\leq\sqrt V\|v\|_2$.
\end{proof}

\begin{lemma}[Covariance-map Lipschitz continuity]
\label{lem:cov-lip}
For any probability vectors $p,q\in\R^V$,
\[
\|C_p-C_q\|_2\leq3\|p-q\|_2.
\]
\end{lemma}

\begin{proof}
Using
\[
C_p-C_q
=
\diag(p-q)-(pp^\top-qq^\top)
\]
and the decomposition
\[
pp^\top-qq^\top
=
(p-q)p^\top+q(p-q)^\top,
\]
we obtain
\[
\begin{aligned}
\|C_p-C_q\|_2
&\leq
\|\diag(p-q)\|_2
+\|pp^\top-qq^\top\|_2\\
&\leq
\|p-q\|_2
+\|p-q\|_2(\|p\|_2+\|q\|_2)\\
&\leq3\|p-q\|_2,
\end{aligned}
\]
where the last inequality uses
$\|p\|_2,\|q\|_2\leq1$.
\end{proof}

\begin{lemma}[Matrix-valued expectation bound]
\label{lem:tv-matrix}
Let $\mu_1,\mu_2$ be probability measures on a finite
or countable set $\mathcal Y$, and let
$M:\mathcal Y\to\R^{n_1\times n_2}$
satisfy $\|M(y)\|_2\leq B$ for all $y\in\mathcal Y$.
Then
\[
\big\|\E_{\mu_1}[M]-\E_{\mu_2}[M]\big\|_2
\leq B\|\mu_1-\mu_2\|_1.
\]
\end{lemma}

\begin{proof}
By linearity of expectation and the triangle inequality,
\[
\begin{aligned}
\big\|\E_{\mu_1}[M]-\E_{\mu_2}[M]\big\|_2
&=
\left\|
\sum_{y\in\mathcal Y}
M(y)\bigl(\mu_1(y)-\mu_2(y)\bigr)
\right\|_2\\
&\leq
\sum_{y\in\mathcal Y}
\|M(y)\|_2\,|\mu_1(y)-\mu_2(y)|\\
&\leq
B\|\mu_1-\mu_2\|_1.
\end{aligned}
\]
\end{proof}
 
\subsection{Proof of the autoregressive chain rule}
\label{app:chain-rule}
 
\begin{proposition}[Eq.~\eqref{eq:kl-effort}]
For any two autoregressive policies $\pi_u, \pi_{\text{ref}}$ over responses $y = (y_1,\dots,y_T)$ given prompt $x$,
\begin{equation}
\label{eq:chain-rule-appendix}
\KL\!\big(\pi_u(\cdot\mid x)\,\|\,\pi_{\text{ref}}(\cdot\mid x)\big)
\;=\; \sum_{t=1}^T \E_{y_{<t}\sim\pi_u}\!\Big[\KL\!\big(\pi_u(\cdot\mid x,y_{<t})\,\|\,\pi_{\text{ref}}(\cdot\mid x,y_{<t})\big)\Big].
\end{equation}
\end{proposition}
 
\begin{proof}
By the autoregressive factorization, $\log \pi_u(y\mid x) = \sum_{t=1}^T \log \pi_u(y_t\mid x, y_{<t})$ and similarly for $\pi_{\text{ref}}$. Therefore
\begin{align}
\KL(\pi_u(\cdot\mid x)\,\|\,\pi_{\text{ref}}(\cdot\mid x))
&= \sum_{y} \pi_u(y\mid x)\log \frac{\pi_u(y\mid x)}{\pi_{\text{ref}}(y\mid x)}
= \sum_{y}\pi_u(y\mid x)\sum_{t=1}^T \log\frac{\pi_u(y_t\mid x,y_{<t})}{\pi_{\text{ref}}(y_t\mid x,y_{<t})} \nonumber\\
&= \sum_{t=1}^T \sum_y \pi_u(y\mid x)\log\frac{\pi_u(y_t\mid x,y_{<t})}{\pi_{\text{ref}}(y_t\mid x,y_{<t})}.
\end{align}
The $t$-th summand depends on $y$ only through $y_{1:t}$. Marginalizing $y_{t+1:T}$ and using $\pi_u(y_{1:t}\mid x) = \pi_u(y_{<t}\mid x)\,\pi_u(y_t\mid x,y_{<t})$,
\begin{align*}
\sum_y \pi_u(y\mid x)\log\frac{\pi_u(y_t\mid x,y_{<t})}{\pi_{\text{ref}}(y_t\mid x,y_{<t})}
&= \sum_{y_{<t}}\pi_u(y_{<t})\sum_{y_t}\pi_u(y_t\mid y_{<t})\log\tfrac{\pi_u(y_t\mid y_{<t})}{\pi_{\text{ref}}(y_t\mid y_{<t})}\\
&= \E_{y_{<t}\sim\pi_u}\!\big[\KL(\pi_u(\cdot\mid y_{<t})\,\|\,\pi_{\text{ref}}(\cdot\mid y_{<t}))\big],
\end{align*}
where we have suppressed the conditioning on $x$ for brevity.
Summing over $t$ yields~\eqref{eq:chain-rule-appendix}.
\end{proof}

\subsection{Proof of Proposition~\ref{prop:kl-quad}}
\label{app:kl-quad-proof}

Fix two steering vectors $u, v$ at position $t$ (suppressing throughout this proof both the position subscript on the steering vectors and the position superscript on the per-position distributions, so we write $u, v$ for $u_t, v_t$ and $\pi_u, \pi_v$ for $\pi_u^t, \pi_v^t$), let $z_u, z_v$ denote the corresponding pre-softmax logits, and define $\delta := z_u - z_v = W(u-v)$. Consider the affine path
\[
z_\tau \;:=\; z_v + \tau\delta, \qquad \pi_\tau \;:=\; \softmax(z_\tau), \qquad \tau \in [0,1],
\]
so that $\pi_0 = \pi_v$ and $\pi_1 = \pi_u$. Define
\[
g(\tau) \;:=\; \KL(\pi_\tau \,\|\, \pi_v),
\]
which satisfies $g(0) = 0$ and $g(1) = \KL(\pi_u\|\pi_v)$. 
Using $\log\pi_\tau(i) - \log\pi_v(i) = \tau\,\delta_i - [A(z_\tau) - A(z_v)]$,
\begin{equation}
\label{eq:g-explicit}
g(\tau) \;=\; \tau\,\langle \pi_\tau, \delta\rangle - A(z_\tau) + A(z_v).
\end{equation}
Differentiating, with $\nabla A(z_\tau) = \pi_\tau$ and $\tfrac{d\pi_\tau}{d\tau} = C_{\pi_\tau}\delta$ (Lemma~\ref{lem:logpart}),
\[
g'(\tau) \;=\; \langle \pi_\tau, \delta\rangle + \tau\,\delta^\top C_{\pi_\tau}\delta - \langle \pi_\tau, \delta\rangle \;=\; \tau\,\delta^\top C_{\pi_\tau}\,\delta.
\]
By the fundamental theorem of calculus,
\begin{equation}
\label{eq:kl-integral}
\KL(\pi_u\,\|\,\pi_v) \;=\; g(1) - g(0) \;=\; \int_0^1 \tau\,\delta^\top C_{\pi_\tau}\,\delta\, d\tau.
\end{equation}

Write $C_{\pi_\tau} = C_{\pi_v} + (C_{\pi_\tau} - C_{\pi_v})$. The first part gives the leading term:
\[
\int_0^1 \tau\,\delta^\top C_{\pi_v}\,\delta\, d\tau \;=\; \tfrac{1}{2}\delta^\top C_{\pi_v}\,\delta \;=\; \tfrac{1}{2}(u-v)^\top F_t(v)\,(u-v),
\]
using $\delta = W(u-v)$ and $F_t(v) = W^\top C_{\pi_v} W$. The second gives a remainder $R$ with
\[
|R| \;=\; \Big|\int_0^1 \tau\,\delta^\top (C_{\pi_\tau} - C_{\pi_v})\,\delta\, d\tau\Big| \;\leq\; \|\delta\|_2^2 \int_0^1 \tau\,\|C_{\pi_\tau} - C_{\pi_v}\|_2\, d\tau.
\]
By Lemma~\ref{lem:cov-lip}, $\|C_{\pi_\tau} - C_{\pi_v}\|_2 \leq 3\|\pi_\tau - \pi_v\|_2 \leq 3\tau\|\delta\|_2$ (Lemma~\ref{lem:softmax-lip}). Hence
\[
|R| \;\leq\; 3\|\delta\|_2^3 \int_0^1 \tau^2\, d\tau \;=\; \|\delta\|_2^3 \;\leq\; \sigma_{\max}(W)^3\,\|u-v\|_2^3.
\]
Thus, for arbitrary steering vectors $u, v$,
\[
\KL(\pi_u\,\|\,\pi_v) \;=\; \tfrac{1}{2}(u-v)^\top F_t(v)\,(u-v) \;+\; O(\|u-v\|^3).
\]
Specializing to $v = 0$ yields $\pi_v = \pi_{\text{ref}}$ and $F_t(0) = W^\top C_{\pi_{\text{ref}}} W$, proving Proposition~\ref{prop:kl-quad}.

\subsection{Proof of the non-asymptotic KL bound}
\label{app:kl-bound}
 
We establish the global upper bound used in Section~\ref{sec:kl-fisher}.
 
\begin{proposition}
For any $u_t\in\R^d$,
\begin{equation}
\label{eq:kl-global-bound}
\KL\!\big(\pi_u(\cdot\mid x,y_{<t})\,\|\,\pi_{\text{ref}}(\cdot\mid x,y_{<t})\big)
\;\leq\; \tfrac{1}{2}\|W u_t\|_2^2
\;\leq\; \tfrac{1}{2}\sigma_{\max}(W)^2\,\|u_t\|_2^2.
\end{equation}
\end{proposition}
 
\begin{proof}
By~\eqref{eq:kl-integral}, $\KL(\pi_u^t\,\|\,\pi_{\text{ref}}^t) = \int_0^1 \tau\,\delta^\top C_{\pi_\tau}\delta\, d\tau$. 
Since $C_{\pi_\tau} \preceq I_V$ (Lemma~\ref{lem:sigma}), $\delta^\top C_{\pi_\tau}\delta \leq \|\delta\|_2^2$. Therefore
\[
\KL(\pi_u^t\,\|\,\pi_{\text{ref}}^t) \;\leq\; \|\delta\|_2^2 \int_0^1 \tau\,d\tau \;=\; \tfrac{1}{2}\|\delta\|_2^2,
\]
and $\|\delta\|_2 = \|Wu_t\|_2 \leq \sigma_{\max}(W)\|u_t\|_2$.
\end{proof}
 
\textbf{Relation to the Fisher quadratic.} Both $\KL(\pi_u^t\,\|\,\pi_{\text{ref}}^t)$ and $\tfrac{1}{2}u_t^\top F_t(0)u_t$ are bounded above by $\tfrac{1}{2}\|\delta\|_2^2$, since $u_t^\top F_t(0) u_t = \delta^\top C_{\pi_{\text{ref}}^t}\delta \leq \|\delta\|_2^2$. The Fisher quadratic is a \textit{local} second-order proxy for the KL (Appendix~\ref{app:kl-quad-proof}) but does not dominate it globally; the simpler $\tfrac{1}{2}\|Wu_t\|^2$ upper bound in~\eqref{eq:kl-global-bound} is the global surrogate.

\subsection{Proof of Theorem \ref{thm:seqkl-grad}}
\label{appendix:seqkl-grad}

\begin{proof}

    Let $\ell(y) := \sum_{\tau=1}^T \log\tfrac{\pi_u(y_\tau\mid y_{<\tau})}{\pi_{\text{ref}}(y_\tau \mid y_{<\tau})}$ and $\mathcal{K}(u) := \E_{y\sim\pi_u}[\ell(y)]$. We derive~\eqref{eq:seqkl-grad} by combining the autoregressive chain rule with direct differentiation of the per-token KL.
 
\textbf{Step 1: autoregressive decomposition.}
By the chain rule (Appendix~\ref{app:chain-rule}),
\begin{align}
\label{eq:K-chain}
\mathcal{K}(u) \;=\;& \sum_{\tau=1}^T \E_{y_{<\tau}\sim\pi_u}\!\big[\mathcal{K}_\tau(u_\tau; y_{<\tau})\big], \\
\mathcal{K}_\tau(u_\tau; y_{<\tau}) \;:=\;& \KL\!\big(\pi_u(\cdot\mid y_{<\tau})\,\|\,\pi_{\text{ref}}(\cdot\mid y_{<\tau})\big). \nonumber
\end{align}
The per-token KL $\mathcal{K}_\tau(u_\tau; y_{<\tau})$ depends on $u_t$ only when $\tau=t$ (through the steering vector $u_t$ itself); for $\tau\neq t$, $\mathcal{K}_\tau$ depends on $u_\tau$ and the context, not on $u_t$. The outer expectation, however, is taken under $\pi_u$, which depends on the entire prefix $u_{<\tau}$.
 
\textbf{Step 2: differentiating~\eqref{eq:K-chain} in $u_t$.}
Using the log-derivative identity $\nabla_{u_t}\pi_u(y_{<\tau}\mid x) = \pi_u(y_{<\tau}\mid x)\,\nabla_{u_t}\log\pi_u(y_{<\tau}\mid x)$:
\begin{align}
\nabla_{u_t}\mathcal{K}(u)
\;=\;& \sum_{\tau=1}^T \E_{y_{<\tau}\sim\pi_u}\!\big[\nabla_{u_t}\mathcal{K}_\tau(u_\tau; y_{<\tau})\big] \nonumber\\
&+ \sum_{\tau=1}^T \E_{y_{<\tau}\sim\pi_u}\!\big[\big(\nabla_{u_t}\log\pi_u(y_{<\tau}\mid x)\big)\,\mathcal{K}_\tau(u_\tau; y_{<\tau})\big].
\label{eq:K-grad-two}
\end{align}
The first piece collects the \textit{explicit-dependence} contributions: by Step~1, only $\tau=t$ contributes, and for $\tau=t$ the inner gradient $\nabla_{u_t}\mathcal{K}_t$ is computed at fixed context (treating $y_{<t}$ as non-random). The second piece collects the \textit{trajectory} contributions from $\pi_u$ depending on $u_t$.
 
\textbf{Step 3: the explicit-dependence term equals $F_t(u_t)u_t$.}
At fixed $y_{<t}$, the closed form~\eqref{eq:per-token-kl-exact} reads $\mathcal{K}_t = A(z_t) - A(z_t+\delta_t) + \langle \pi_u^t,\delta_t\rangle$ with $\delta_t = Wu_t$. Differentiating using $\nabla A(z_t+\delta_t) = \pi_u^t$ and $\nabla_{\delta_t}\pi_u^t = C_{\pi_u^t}$ (Lemma~\ref{lem:logpart}),
\[
\nabla_{\delta_t}\mathcal{K}_t
\;=\; -\pi_u^t + \pi_u^t + C_{\pi_u^t}\,\delta_t
\;=\; C_{\pi_u^t}\,\delta_t,
\qquad
\nabla_{u_t}\mathcal{K}_t
\;=\; W^\top C_{\pi_u^t} W\, u_t
\;=\; F_t(u_t)\, u_t,
\]
where the cancellation in the first equation comes from the product rule on $\langle\pi_u^t,\delta_t\rangle$ and the identification $F_t(u_t) := W^\top C_{\pi_u^t} W$ matches~\eqref{eq:fisher-u}. The relation
\begin{equation}
\label{eq:inner-grad-Fu}
\nabla_{u_t}\mathcal{K}_t(u_t; y_{<t}) \;=\; F_t(u_t)\, u_t
\end{equation}
holds \textit{exactly} (no truncation error), since the per-position KL is a Bregman divergence whose gradient is the Hessian of $A$ at $z+\delta$ applied to $\delta$. Taking expectation over $y_{<t}\sim\pi_u$ gives the analytic term in~\eqref{eq:seqkl-grad}.
 
\textbf{Step 4: the trajectory term is the suffix score-function term.}
We show that the second line of~\eqref{eq:K-grad-two}, which we abbreviate as
\begin{equation}
\label{eq:traj-term}
\mathcal{T}(u) \;:=\; \sum_{\tau=1}^T \E_{y_{<\tau}\sim\pi_u(\cdot\mid x)}\!\Big[\big(\nabla_{u_t}\log\pi_u(y_{<\tau}\mid x)\big)\,\mathcal{K}_\tau(u_\tau; y_{<\tau})\Big],
\end{equation}
equals the suffix score-function term $\E_{y\sim\pi_u(\cdot\mid x)}\!\big[\big(\nabla_{u_t}\log\pi_u(y_t\mid x, y_{<t})\big)\,\ell_{>t}(y)\big]$ in~\eqref{eq:seqkl-grad}. The argument has three parts: (a) the trajectory score collapses to a single position-$t$ contribution that vanishes for $\tau \le t$; (b) the score factors out of the inner conditional expectation; and (c) the remaining inner sum is exactly the conditional sequence-KL of the suffix.

\textbf{(a) Reducing the trajectory score to position $t$.}
By the autoregressive factorization of $\pi_u(\cdot\mid x)$,
\[
\log\pi_u(y_{<\tau}\mid x) \;=\; \sum_{t'<\tau}\log\pi_u(y_{t'}\mid x, y_{<t'}),
\]
hence
\[
\nabla_{u_t}\log\pi_u(y_{<\tau}\mid x) \;=\; \sum_{t'<\tau}\nabla_{u_t}\log\pi_u(y_{t'}\mid x, y_{<t'}).
\]
Each per-token log-probability admits the explicit form
\[
\log\pi_u(y_{t'}\mid x, y_{<t'}) \;=\; \big(W(h_{t'}+u_{t'})+b\big)_{y_{t'}} \;-\; A\!\big(W(h_{t'}+u_{t'})+b\big),
\]
where $h_{t'} = h_{t'}(x, y_{<t'})$ is the frozen base-model hidden state. Because steering is applied only at the pre-logit stage, $h_{t'}$ does not depend on $u$; therefore $\log\pi_u(y_{t'}\mid x, y_{<t'})$ depends on $u_t$ if and only if the term $u_{t'}$ appearing inside $W(h_{t'}+u_{t'})$ is $u_t$, i.e., if and only if $t' = t$. Consequently
\[
\nabla_{u_t}\log\pi_u(y_{t'}\mid x, y_{<t'}) \;=\; \begin{cases} \nabla_{u_t}\log\pi_u(y_t\mid x, y_{<t}), & t' = t, \\[2pt] 0, & t' \neq t, \end{cases}
\]
and substituting into the expansion above,
\begin{equation}
\label{eq:score-collapse}
\nabla_{u_t}\log\pi_u(y_{<\tau}\mid x) \;=\; \nabla_{u_t}\log\pi_u(y_t\mid x, y_{<t})\cdot \mathbf{1}[\tau > t].
\end{equation}
Every summand in~\eqref{eq:traj-term} with $\tau \le t$ vanishes, and the trajectory term reduces to
\begin{equation}
\label{eq:traj-restricted}
\mathcal{T}(u) \;=\; \sum_{\tau>t} \E_{y_{<\tau}\sim\pi_u(\cdot\mid x)}\!\Big[\big(\nabla_{u_t}\log\pi_u(y_t\mid x, y_{<t})\big)\,\mathcal{K}_\tau(u_\tau; y_{<\tau})\Big].
\end{equation}

\textbf{(b) Factoring the score out of the inner expectation.}
The score $\nabla_{u_t}\log\pi_u(y_t\mid x, y_{<t})$ is a function of $(x, y_{\le t})$ alone. We separate the $\tau = t+1$ summand of~\eqref{eq:traj-restricted}, for which $y_{<\tau} = y_{\le t}$ and no further integration over the suffix is required, from the remaining summands with $\tau \ge t+2$, for which $y_{<\tau} = (y_{\le t}, y_{t+1:\tau-1})$ contains nontrivial random suffix tokens $y_{t+1:\tau-1}$:
\begin{align*}
\mathcal{T}(u) \;=\;& \E_{y_{\le t}\sim\pi_u(\cdot\mid x)}\!\Big[\big(\nabla_{u_t}\log\pi_u(y_t\mid x, y_{<t})\big)\,\mathcal{K}_{t+1}(u_{t+1}; y_{\le t})\Big] \\
&+ \sum_{\tau = t+2}^{T} \E_{y_{<\tau}\sim\pi_u(\cdot\mid x)}\!\Big[\big(\nabla_{u_t}\log\pi_u(y_t\mid x, y_{<t})\big)\,\mathcal{K}_\tau(u_\tau; y_{<\tau})\Big].
\end{align*}
For each $\tau \ge t+2$, splitting the expectation over $y_{<\tau}$ under $\pi_u(\cdot\mid x)$ into an outer expectation over $y_{\le t}$ and an inner conditional expectation over $y_{t+1:\tau-1}$ given $y_{\le t}$, and using that the score is measurable with respect to $y_{\le t}$ to factor it out of the inner expectation,
\begin{align*}
\E_{y_{<\tau}\sim\pi_u(\cdot\mid x)}\!\Big[\big(\nabla_{u_t}\log\pi_u(y_t\mid x, y_{<t})\big)\,\mathcal{K}_\tau(u_\tau; y_{<\tau})\Big]
&\;=\;
\\
& \E_{y_{\le t}\sim\pi_u(\cdot\mid x)}\!\bigg[\big(\nabla_{u_t}\log\pi_u(y_t\mid x, y_{<t})\big)\\
&\times\, \E_{y_{t+1:\tau-1}\sim\pi_u(\cdot\mid x, y_{\le t})}\!\big[\mathcal{K}_\tau(u_\tau; y_{<\tau})\big]\bigg].
\end{align*}
Summing over $\tau \ge t+2$, combining with the $\tau = t+1$ term, and exchanging the finite sum with the outer expectation,
\begin{equation}
\label{eq:traj-after-tower}
\mathcal{T}(u) \;=\; \E_{y_{\le t}\sim\pi_u(\cdot\mid x)}\!\Big[\big(\nabla_{u_t}\log\pi_u(y_t\mid x, y_{<t})\big)\,S(y_{\le t}; u)\Big],
\end{equation}
where we have introduced the auxiliary quantity
\begin{equation}
\label{eq:S-def}
S(y_{\le t}; u) \;:=\; \mathcal{K}_{t+1}(u_{t+1}; y_{< t + 1}) \;+\; \sum_{\tau = t+2}^{T} \E_{y_{t+1:\tau-1}\sim\pi_u(\cdot\mid x, y_{\le t})}\!\big[\mathcal{K}_\tau(u_\tau; y_{<\tau})\big].
\end{equation}

\textbf{(c) Identifying $S(y_{\le t}; u)$ as the conditional suffix KL.}
Fix $y_{\le t}$ and apply the autoregressive chain rule of Appendix~\ref{app:chain-rule} to the conditional sequence laws $\pi_u(\cdot\mid x, y_{\le t})$ and $\pi_{\text{ref}}(\cdot\mid x, y_{\le t})$ over the suffix variable $y_{>t} = (y_{t+1}, \ldots, y_T)$. The first chain-rule term (corresponding to the absolute position $\tau = t+1$) is the per-position KL at position $t+1$ given $y_{\le t}$, with no further conditional averaging; the remaining terms ($\tau \ge t+2$) involve genuine conditional expectations over intermediate suffix tokens. We have
\begin{align}
    \begin{split}
        \label{eq:suffix-chain}
\KL\!\big(\pi_u(\cdot\mid x, y_{\le t})\|\pi_{\text{ref}}(\cdot\mid x, y_{\le t})\big)
&\!=\! \mathcal{K}_{t+1}(u_{t+1};y_{< t + 1}) \\
&\!+\! \sum_{\tau = t+2}^{T} \! \E_{y_{t+1:\tau-1}\sim\pi_u(\cdot\mid x, y_{\le t})}\!\big[\mathcal{K}_\tau(u_\tau; y_{<\tau})\big]
\;=\; S(y_{\le t}; u).
    \end{split}
\end{align}
On the other hand, expanding the same conditional KL directly as an expectation of a sum of log-ratios under $\pi_u(\cdot\mid x, y_{\le t})$,
\begin{align}
    \begin{split}
        \label{eq:suffix-as-expectation}
\KL\!\big(\pi_u(\cdot\mid x, y_{\le t})\,\|\,\pi_{\text{ref}}(\cdot\mid x, y_{\le t})\big)
\;&=\; \E_{y_{>t}\sim\pi_u(\cdot\mid x, y_{\le t})}\!\Big[\sum_{\tau > t}\log\tfrac{\pi_u(y_\tau\mid x, y_{<\tau})}{\pi_{\text{ref}}(y_\tau\mid x, y_{<\tau})}\Big]
\\
&=\; \E_{y_{>t}\sim\pi_u(\cdot\mid x, y_{\le t})}\!\big[\ell_{>t}(y)\big],
    \end{split}
\end{align}
where the last equality uses $\ell_{>t}(y) = \sum_{\tau > t}\log\tfrac{\pi_u(y_\tau\mid x, y_{<\tau})}{\pi_{\text{ref}}(y_\tau\mid x, y_{<\tau})}$ (and note that $\ell_{>t}(y)$ depends on $y$ only through its suffix $y_{>t}$ once $y_{\le t}$ is fixed). Combining~\eqref{eq:suffix-chain} and~\eqref{eq:suffix-as-expectation},
\begin{equation}
\label{eq:S-final}
S(y_{\le t}; u) \;=\; \E_{y_{>t}\sim\pi_u(\cdot\mid x, y_{\le t})}\!\big[\ell_{>t}(y)\big].
\end{equation}

\textbf{Combining the pieces.}
Substituting~\eqref{eq:S-final} into~\eqref{eq:traj-after-tower} and merging the outer expectation over $y_{\le t}\sim\pi_u(\cdot\mid x)$ with the inner expectation over $y_{>t}\mid y_{\le t}$ into a single expectation over $y\sim\pi_u(\cdot\mid x)$, we have
\begin{equation}
\label{eq:traj-final}
\mathcal{T}(u) \;=\; \E_{y\sim\pi_u(\cdot\mid x)}\!\Big[\big(\nabla_{u_t}\log\pi_u(y_t\mid x, y_{<t})\big)\,\ell_{>t}(y)\Big].
\end{equation}
This is the suffix score-function term in Theorem~\ref{thm:seqkl-grad}. Combining~\eqref{eq:traj-final} with the analytic term established in Step~3 yields~\eqref{eq:seqkl-grad} and completes the proof.
\end{proof}

\subsection{Proof of Theorem~\ref{thm:surrogate-equiv}}
\label{app:proofOfSurrogateEquiv}

\begin{proof}
Let $B_W:=\sigma_{\max}(W)$ and, for any probability vector
$p\in\R^V$, define
\(
C_p:=\diag(p)-pp^\top.
\)
Recall that, at a fixed prefix $y_{<t}$,
\[
F_t(u_t)=W^\top C_{\pi_u^t}W,
\qquad
F_t(0)=W^\top C_{\pi_{\mathrm{ref}}^t}W.
\]
We establish the three bounds separately and then combine them.

\paragraph{Step 1: Perturbation of the token-level Fisher matrix.}
For any probability vectors $p,q\in\R^V$, using Lemma \ref{lem:cov-lip}, 
\begin{equation}
\label{eq:proof-cov-lipschitz}
\begin{aligned}
\|C_p-C_q\|_2 \le 3\|p-q\|_2.
\end{aligned}
\end{equation}
Moreover, the Jacobian of the softmax is $C_p$, with
$\|C_p\|_2\le1$. The mean-value integral formula therefore gives
\begin{equation}
\label{eq:proof-softmax-lipschitz}
\begin{aligned}
\|\pi_u^t-\pi_{\mathrm{ref}}^t\|_2
&\le \|Wu_t\|_2\\
&\le B_W\|u_t\|_2.
\end{aligned}
\end{equation}
Combining \eqref{eq:proof-cov-lipschitz} and
\eqref{eq:proof-softmax-lipschitz}, we obtain
\begin{equation}
\label{eq:proof-fisher-lipschitz}
\begin{aligned}
\|F_t(u_t)-F_t(0)\|_2
&\le
B_W^2\|C_{\pi_u^t}-C_{\pi_{\mathrm{ref}}^t}\|_2\\
&\le
3B_W^3\|u_t\|_2.
\end{aligned}
\end{equation}
Since this bound holds uniformly over prefixes, taking
expectations under $\pi_u$ yields
\begin{equation}
\label{eq:proof-fisher-onoff}
\begin{aligned}
\|\tilde F_t^u-\bar F_t^u\|_2
&=
\left\|
\E_{y_{<t}\sim\pi_u}
[F_t(u_t)-F_t(0)]
\right\|_2\\
&\le
3B_W^3\|u_t\|_2
=C_1\|u_t\|_2.
\end{aligned}
\end{equation}

\paragraph{Step 2: Perturbation of the prefix distribution.}
Since $\|C_p\|_2\le1$, we have
$\|F_t(0)\|_2\le B_W^2$ for every prefix. Thus,
\begin{equation}
\label{eq:proof-fisher-prefix}
\begin{aligned}
\|\bar F_t^u-\bar F_t\|_2
&=
\left\|
\E_{y_{<t}\sim\pi_u}[F_t(0)]
-
\E_{y_{<t}\sim\pi_{\mathrm{ref}}}[F_t(0)]
\right\|_2\\
&\le
B_W^2
\|\pi_u(y_{<t}\mid x)
-\pi_{\mathrm{ref}}(y_{<t}\mid x)\|_1.
\end{aligned}
\end{equation}

To bound the difference between the prefix distributions, we use
a telescoping argument for autoregressive measures. Let $\mu$
and $\nu$ be two autoregressive distributions over sequences
of length $t-1$. For $s=0,\ldots,t-1$, define the intermediate
distribution
\begin{equation*}
P_s(y_{<t})
:=
\prod_{r=1}^{s}\mu(y_r\mid y_{<r})
\prod_{r=s+1}^{t-1}\nu(y_r\mid y_{<r}),
\end{equation*}
where empty products equal one. By construction,
$P_0=\nu$ and $P_{t-1}=\mu$. Therefore,
\begin{equation*}
\mu-\nu
=
\sum_{s=1}^{t-1}(P_s-P_{s-1}).
\end{equation*}
For each $s$, the two intermediate distributions differ
only in their conditional distribution at position $s$:
\begin{equation*}
\begin{aligned}
&P_s(y_{<t})-P_{s-1}(y_{<t})\\
&\quad=
\mu(y_{<s})
\bigl(
\mu(y_s\mid y_{<s})
-\nu(y_s\mid y_{<s})
\bigr)
\prod_{r=s+1}^{t-1}\nu(y_r\mid y_{<r}).
\end{aligned}
\end{equation*}
Taking absolute values and summing over $y_{<t}$,
the conditional factors corresponding to positions
$r>s$ marginalize to one. Hence,
\begin{equation*}
\begin{aligned}
\|P_s-P_{s-1}\|_1
&=
\sum_{y_{<s}}\mu(y_{<s})
\sum_{y_s}
\left|
\mu(y_s\mid y_{<s})
-\nu(y_s\mid y_{<s})
\right|\\
&=
\E_{y_{<s}\sim\mu}
\left[
\|\mu(\cdot\mid y_{<s})
-\nu(\cdot\mid y_{<s})\|_1
\right].
\end{aligned}
\end{equation*}
Applying the triangle inequality to the telescoping sum
establishes
\begin{equation}
\label{eq:proof-autoregressive-tv}
\|\mu-\nu\|_1
\le
\sum_{s<t}
\E_{y_{<s}\sim\mu}
\left[
\|\mu(\cdot\mid y_{<s})
-\nu(\cdot\mid y_{<s})\|_1
\right].
\end{equation}

Applying \eqref{eq:proof-autoregressive-tv} to
$\mu=\pi_u$ and $\nu=\pi_{\mathrm{ref}}$, and using
$\|v\|_1\le\sqrt V\|v\|_2$, gives
\begin{equation}
\label{eq:proof-prefix-distance}
\begin{aligned}
&\|\pi_u(y_{<t}\mid x)
-\pi_{\mathrm{ref}}(y_{<t}\mid x)\|_1\\
&\quad\le
\sum_{\tau<t}
\E_{y_{<\tau}\sim\pi_u}
\left[
\|\pi_u^\tau-\pi_{\mathrm{ref}}^\tau\|_1
\right]\\
&\quad\le
\sqrt V B_W\sum_{\tau<t}\|u_\tau\|_2,
\end{aligned}
\end{equation}
where the last step follows from
\eqref{eq:proof-softmax-lipschitz}.
Substituting \eqref{eq:proof-prefix-distance} into
\eqref{eq:proof-fisher-prefix} yields
\begin{equation}
\label{eq:proof-fisher-offoff}
\|\bar F_t^u-\bar F_t\|_2
\le
\sqrt V B_W^3\sum_{\tau<t}\|u_\tau\|_2
=
C_2\sum_{\tau<t}\|u_\tau\|_2.
\end{equation}
This completes the proof of \eqref{eq:surrogate-bounds}.

\paragraph{Step 3: Bounding the suffix score-function term.}
Define the position-$t$ score
\begin{equation*}
s_t(y;u)
:=
\nabla_{u_t}\log\pi_u(y_t\mid x,y_{<t})
=
W^\top(e_{y_t}-\pi_u^t).
\end{equation*}
Because $\|e_{y_t}\|_2=1$ and $\|\pi_u^t\|_2\le1$,
\begin{equation}
\label{eq:proof-score-norm}
\|s_t(y;u)\|_2\le2B_W.
\end{equation}

For $\tau>t$, let
\[
r_\tau(y;u)
:=
\log
\frac{\pi_u(y_\tau\mid x,y_{<\tau})}
{\pi_{\mathrm{ref}}(y_\tau\mid x,y_{<\tau})}.
\]
The suffix score-function term in
Theorem~\ref{thm:seqkl-grad} can be written as
\begin{equation*}
g_t^{\mathrm{suffix}}(u)
=
\sum_{\tau=t+1}^T
\E_{y\sim\pi_u}[s_t(y;u)r_\tau(y;u)].
\end{equation*}

For $\tau>t$, the score $s_t(y;u)$ depends only on
$y_{\le t}$ and is therefore measurable with respect
to $y_{<\tau}$. Conditioning on this prefix gives
\begin{equation*}
\begin{aligned}
&\E_{y\sim\pi_u}[s_t(y;u)r_\tau(y;u)]\\
&\quad=
\E_{y_{<\tau}\sim\pi_u}
\left[
s_t(y;u)
\E_{y_\tau\sim\pi_u^\tau}
[r_\tau(y;u)\mid y_{<\tau}]
\right]\\
&\quad=
\E_{y_{<\tau}\sim\pi_u}
\left[
s_t(y;u)
\KL(\pi_u^\tau\|\pi_{\mathrm{ref}}^\tau)
\right].
\end{aligned}
\end{equation*}
Consequently,
\begin{equation}
\label{eq:proof-suffix-conditional}
g_t^{\mathrm{suffix}}(u)
=
\sum_{\tau=t+1}^T
\E_{y_{<\tau}\sim\pi_u}
\left[
s_t(y;u)
\KL(\pi_u^\tau\|\pi_{\mathrm{ref}}^\tau)
\right].
\end{equation}

By the global token-level KL bound established in
Appendix~\ref{app:kl-bound},
\begin{equation}
\label{eq:proof-kl-global}
0\le
\KL(\pi_u^\tau\|\pi_{\mathrm{ref}}^\tau)
\le
\frac12\|Wu_\tau\|_2^2
\le
\frac12 B_W^2\|u_\tau\|_2^2.
\end{equation}
Combining \eqref{eq:proof-score-norm},
\eqref{eq:proof-suffix-conditional}, and
\eqref{eq:proof-kl-global}, we obtain
\begin{equation}
\label{eq:proof-suffix-bound}
\begin{aligned}
\|g_t^{\mathrm{suffix}}(u)\|_2
&\le
\sum_{\tau=t+1}^T
\E_{y_{<\tau}\sim\pi_u}
\left[
\|s_t(y;u)\|_2
\KL(\pi_u^\tau\|\pi_{\mathrm{ref}}^\tau)
\right]\\
&\le
B_W^3\sum_{\tau>t}\|u_\tau\|_2^2.
\end{aligned}
\end{equation}
This proves \eqref{eq:suffix-second-order-bound}.

\paragraph{Step 4: Approximation of the full sequence-level KL gradient.}
By Theorem~\ref{thm:seqkl-grad},
\begin{equation*}
\nabla_{u_t}K(u)
=
\tilde F_t^u u_t+g_t^{\mathrm{suffix}}(u).
\end{equation*}
Adding and subtracting $\bar F_t^u u_t$ gives
\begin{equation*}
\begin{aligned}
\nabla_{u_t}K(u)-\bar F_tu_t
={}&
(\tilde F_t^u-\bar F_t^u)u_t\\
&+
(\bar F_t^u-\bar F_t)u_t
+
g_t^{\mathrm{suffix}}(u).
\end{aligned}
\end{equation*}
Taking norms and applying
\eqref{eq:proof-fisher-onoff},
\eqref{eq:proof-fisher-offoff}, and
\eqref{eq:proof-suffix-bound}, we obtain
\begin{equation*}
\begin{aligned}
&\|\nabla_{u_t}K(u)-\bar F_tu_t\|_2\\
&\quad\le
C_1\|u_t\|_2^2
+
C_2\|u_t\|_2\sum_{\tau<t}\|u_\tau\|_2
+
B_W^3\sum_{\tau>t}\|u_\tau\|_2^2,
\end{aligned}
\end{equation*}
which establishes \eqref{eq:full-gradient-error}.

Finally, define
\[
\|u\|_2^2:=\sum_{\tau=1}^T\|u_\tau\|_2^2.
\]
By Cauchy--Schwarz and Young's inequality,
\begin{equation*}
\begin{aligned}
\|u_t\|_2\sum_{\tau<t}\|u_\tau\|_2
&\le
\sqrt{t-1}\|u_t\|_2
\left(\sum_{\tau<t}\|u_\tau\|_2^2\right)^{1/2}\\
&\le
\frac{\sqrt{t-1}}{2}\|u\|_2^2.
\end{aligned}
\end{equation*}
Since $\|u_t\|_2^2\le\|u\|_2^2$ and
$\sum_{\tau>t}\|u_\tau\|_2^2\le\|u\|_2^2$,
\begin{equation*}
\begin{aligned}
&\|\nabla_{u_t}K(u)-\bar F_tu_t\|_2\\
&\quad\le
\left(
C_1+B_W^3+\frac{\sqrt{t-1}}{2}C_2
\right)\|u\|_2^2.
\end{aligned}
\end{equation*}
For fixed $T$, the coefficient is independent of $u$.
Therefore,
\[
\nabla_{u_t}K(u)
=
\bar F_tu_t+O(\|u\|_2^2),
\qquad u\to0,
\]
which proves \eqref{eq:surrogate-ladder} and completes
the proof.
\end{proof}

\section{Further Experimental Results}

\begin{figure}[ht!]
    \centering
    \includegraphics[width=\linewidth]{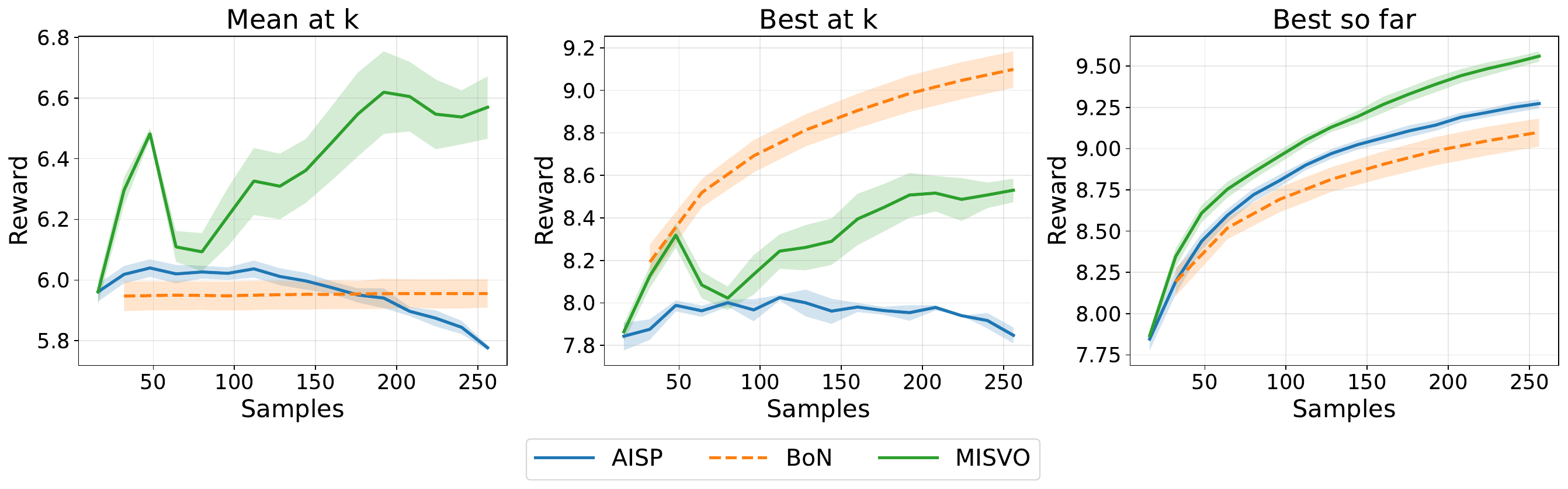}
    \caption{Reward trajectory for Gemma 3 4B/SHP as a function of the cumulative number of sampled completions. \textbf{Left:} mean reward within each step. \textbf{Middle:} maximum reward within each step. \textbf{Right:} cumulative maximum reward.}
    \label{fig:gemmaTrajectoryShp}
\end{figure}

\begin{figure}[ht!]
    \centering
    \includegraphics[width=\linewidth]{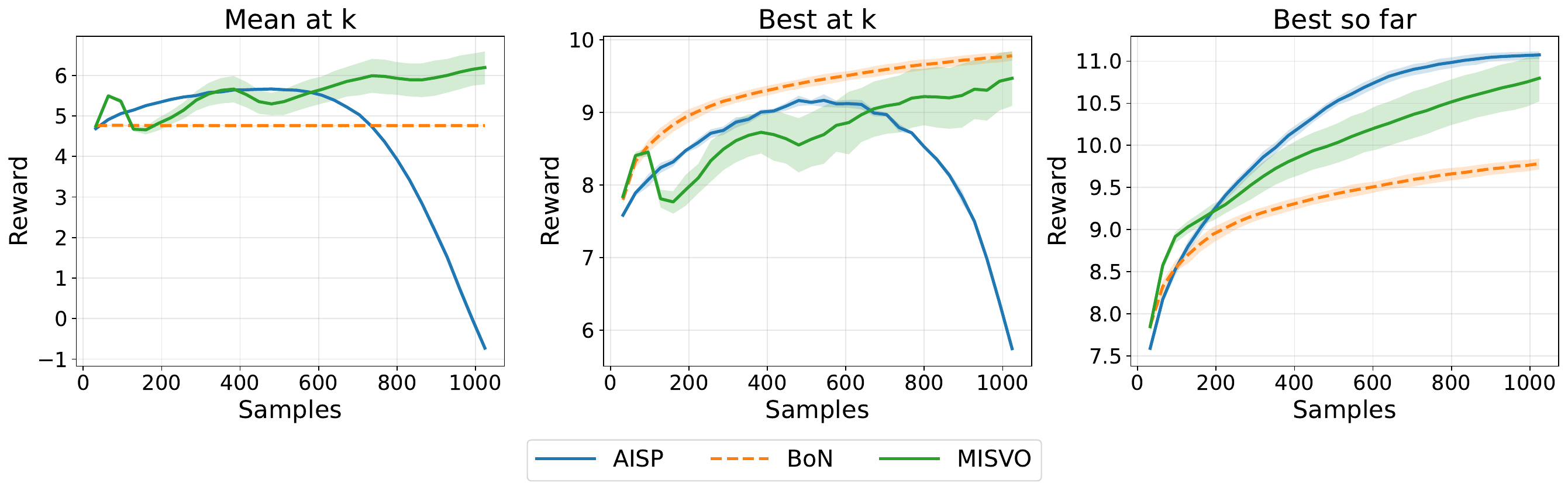}
    \caption{Reward trajectory for Llama 3 8B Instruct/SHP as a function of the cumulative number of sampled completions. \textbf{Left:} mean reward within each step. \textbf{Middle:} maximum reward within each step. \textbf{Right:} cumulative maximum reward.}
    \label{fig:llamaTrajectoryShp}
\end{figure}

\begin{figure}[ht!]
    \centering
    \includegraphics[width=\linewidth]{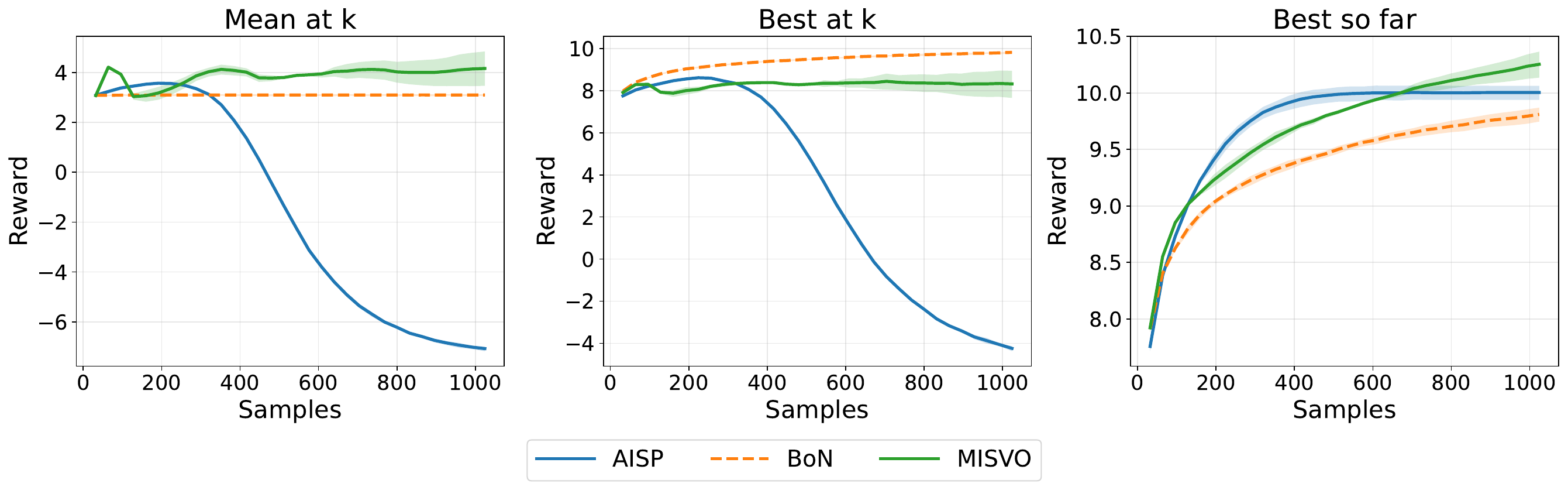}
    \caption{Reward trajectory for Phi 4/SHP as a function of the cumulative number of sampled completions. \textbf{Left:} mean reward within each step. \textbf{Middle:} maximum reward within each step. \textbf{Right:} cumulative maximum reward.}
    \label{fig:phiTrajectoryShp}
\end{figure}

\subsection{Ablation: Fisher surrogates}
\label{sec:exp-ablation}

We compare the Fisher surrogates $\{\bar F_t,\bar F_t^u,\tilde F_t^u\}$ and a KL-regularized variant on LFM2.5-1.2B with SHP, holding the remaining hyperparameters at their main-experiment values.

\input{tables/regularizer_ablation}

Table~\ref{tab:reg-ablation} reports reward differences of at most $0.34$ and diversity/coherence differences of at most $0.013$ among the Fisher variants. The frozen-reference surrogate has the highest reported reward in this setting. The variants' similar performance is compatible with their local first-order equivalence, but does not establish that the experiments remain in the asymptotic regime. The KL-regularized variant has lower reported reward.

\subsection{Runtime and memory}
\label{app:runtime}
We profile each method on two models using one NVIDIA H200 NVL. After loading the language and reward models, we process five warmup prompts followed by ten timed prompts, sequentially. CUDA events bracket each per-prompt invocation, including sampling, reward scoring, optimization, and response selection. Table~\ref{tab:profile} reports per-prompt latency, generation throughput, and peak GPU memory on SHP with a 512-token generation cap. BoN is fastest on both models; AISP and MISVO interleave generation with scoring and updates.

MISVO retains reference hidden states from the initial rollouts and uses vocabulary-sized intermediates for Fisher--vector products. These contribute to its higher peak memory. Chunked products may reduce working memory, but we do not evaluate them here.

The measurements use single-prompt batches: $K$ rollouts per iteration for AISP/MISVO and $KN$ candidates for BoN. Cross-prompt batching may change throughput and is outside this profiling protocol.

\input{tables/profile}

\section{Relation to per-prompt adapter fitting}
Per-prompt low-rank adaptation (LoRA) offers another form of test-time reward optimization. Its cost depends on adapter rank, placement, and the gradients required for the adapted modules. Pre-logit steering confines optimization to additive interventions at the frozen LM head and retains a shared set of model weights.

Our experiments compare candidate selection and pre-logit steering at matched rollout counts. This isolates performance at a common sampling budget but does not equalize wall-clock cost or memory. A comparison with per-prompt LoRA under both rollout and runtime budgets would broaden the evaluation. We leave this comparison to future work and make no claim that steering is preferable to adapter fitting.


\end{document}

%% file: tables/shp_skywork.tex

\begin{table}[t]
\centering
\caption{Results on SHP across model scales. We sample $500$ prompts and score with Skywork-Reward-V2-Qwen3-0.6B. Reward (best-so-far over the optimization trace) is the primary metric, reported as mean $\pm$ std over $3$ seeds ($42, 1340, 2026$); \textbf{bold} indicates the best per model. \textsc{MISVO} uses the established  Fisher-quadratic in \eqref{eq:fisher-effort} $\Bar{F}_t$. All methods see a matched per-prompt rollout budget of $K \cdot N$ samples.}
\scalebox{0.9}{
\begin{tabular}{llccc}
\toprule
Model & Method & Reward $\uparrow$ & Diversity $\uparrow$ & Coherence $\uparrow$ \\
\midrule
\multirow{3}{*}{LFM2.5 1.2B}
  & BoN (top-$p$) & $8.70 \pm 0.11$ & $0.688 \pm 0.004$ & $0.671 \pm 0.004$ \\
  & AISP          & $9.09 \pm 0.03$ & $0.726 \pm 0.015$ & $0.663 \pm 0.009$ \\
  & MISVO          & $\mathbf{9.29 \pm 0.04}$ & $0.719 \pm 0.002$ & $0.654 \pm 0.005$ \\
\midrule
\multirow{3}{*}{Gemma3 4B}
  & BoN (top-$p$) & $9.10 \pm 0.11$ & $0.721 \pm 0.002$ & $0.632 \pm 0.006$ \\
  & AISP          & $9.27 \pm 0.04$ & $0.741 \pm 0.005$ & $0.631 \pm 0.005$ \\
  & MISVO          & $\mathbf{9.56 \pm 0.04}$ & $0.728 \pm 0.005$ & $0.632 \pm 0.003$ \\
\midrule
\multirow{3}{*}{Llama3 8B Instruct}
  & BoN (top-$p$) & $9.78 \pm 0.08$ & $0.700 \pm 0.002$ & $0.671 \pm 0.003$ \\
  & AISP          & $\mathbf{11.07 \pm 0.06}$ & $0.757 \pm 0.002$ & $0.665 \pm 0.001$ \\
  & MISVO          & $10.80 \pm 0.35$ & $0.716 \pm 0.010$ & $0.672 \pm 0.003$ \\
\midrule
\multirow{3}{*}{Phi 4 (14B)}
  & BoN (top-$p$) & $9.81 \pm 0.08$ & $0.718 \pm 0.004$ & $0.636 \pm 0.005$ \\
  & AISP          & $10.00 \pm 0.08$ & $0.769 \pm 0.002$ & $0.634 \pm 0.005$ \\
  & MISVO          & $\mathbf{10.25 \pm 0.14}$ & $0.735 \pm 0.008$ & $0.636 \pm 0.001$ \\
\bottomrule
\end{tabular}
}
\label{tab:shp}
\end{table}

%% file: tables/mbpp.tex

\begin{table}[t]
\centering
\caption{Results on MBPP+ across model scales. Metric is the pass rate (percentage of unit tests passed), reported as mean $\pm$ std over $3$ seeds ($42, 1340, 2026$); \textbf{bold} indicates the best per model. AISP is shown at two guidance strengths $\lambda \in \{0.1, 1.0\}$. \textsc{MISVO} uses the established  Fisher-quadratic in \eqref{eq:fisher-effort} $\Bar{F}_t$.}
\scalebox{0.9}{
\begin{tabular}{llccc}
\toprule
Model & Method & Pass rate $\uparrow$ & Diversity $\uparrow$ & Coherence $\uparrow$ \\
\midrule
\multirow{4}{*}{LFM2.5 1.2B}
  & BoN (top-$p$)          & $70.3 \pm 1.3$ & $0.613 \pm 0.018$ & $0.605 \pm 0.010$ \\
  & AISP ($\lambda{=}0.1$) & $67.5 \pm 1.4$ & $0.603 \pm 0.023$ & $0.677 \pm 0.006$ \\
  & AISP ($\lambda{=}1$)   & $66.9 \pm 1.7$ & $0.601 \pm 0.015$ & $0.679 \pm 0.004$ \\
  & MISVO                   & $\mathbf{73.9 \pm 2.7}$ & $0.588 \pm 0.008$ & $0.680 \pm 0.007$ \\
\midrule
\multirow{4}{*}{Gemma3 4B}
  & BoN (top-$p$)          & $73.9 \pm 1.0$ & $0.596 \pm 0.031$ & $0.594 \pm 0.005$ \\
  & AISP ($\lambda{=}0.1$) & $74.2 \pm 0.8$ & $0.610 \pm 0.016$ & $0.593 \pm 0.004$ \\
  & AISP ($\lambda{=}1$)   & $73.3 \pm 0.8$ & $0.616 \pm 0.013$ & $0.591 \pm 0.004$ \\
  & MISVO                   & $\mathbf{76.1 \pm 2.1}$ & $0.596 \pm 0.002$ & $0.586 \pm 0.009$ \\
\midrule
\multirow{4}{*}{Llama3 8B Instruct}
  & BoN (top-$p$)          & $\mathbf{73.3 \pm 1.4}$ & $0.584 \pm 0.015$ & $0.598 \pm 0.010$ \\
  & AISP ($\lambda{=}0.1$) & $70.8 \pm 2.5$ & $0.637 \pm 0.014$ & $0.598 \pm 0.005$ \\
  & AISP ($\lambda{=}1$)   & $70.0 \pm 0.8$ & $0.645 \pm 0.022$ & $0.597 \pm 0.002$ \\
  & MISVO                   & $\mathbf{73.3 \pm 1.4}$ & $0.568 \pm 0.033$ & $0.601 \pm 0.013$ \\
\bottomrule
\end{tabular}
}
\label{tab:mbpp}
\end{table}

%% file: tables/policyDeviation.tex

\begin{table}[htb]
\centering
\caption{Measured policy deviation on SHP with LFM2.5-1.2B. KL is
$\widehat{KL}_\text{ref}(u)$ from~\eqref{eq:kl-estimator};
$\pm$ denotes variation across prompts.}
\label{tab:measured-kl}
\begin{tabular}{lccc}
\toprule
Method / regularizer & Reward $\uparrow$ & Final $\|u\|$ & Final KL $\downarrow$ \\
\midrule
MISVO   & $\mathbf{9.60} \pm 3.38$ & $339.9 \pm 22.4$  & $47.5 \pm 12.9$  \\
AISP                   & $9.28 \pm 3.10$          & $2631.7 \pm 86.3$ & $303.9 \pm 61.6$ \\
Unregularized steering & $9.31 \pm 3.30$          & $474.6 \pm 24.9$  & $162.2 \pm 38.3$ \\
\midrule
Direct KL regularizer  & $8.49 \pm 3.07$          & $459.3 \pm 29.1$  & $120.1 \pm 29.7$  \\
Hidden-space $\ell_2$  & $8.39 \pm 2.86$          & $48.3 \pm 20.4$   & $0.92 \pm 1.1$     \\
\midrule
BoN ($N\!=\!256$)      & $8.89 \pm 2.69$          & n/a               & n/a               \\
\bottomrule
\end{tabular}
\end{table}

%% file: tables/regularizer_ablation.tex
\begin{table}[t]
\centering
\caption{Regularizer ablation on SHP with LFM2.5-1.2B. Settings: $K = N = 16$, $\eta = 0.1$, $\lambda = 1.0$, seed 42, $500$ prompts. 
The three Fisher surrogates lie within close values of each other (Theorem~\ref{thm:surrogate-equiv}).
}
\label{tab:reg-ablation}
\scalebox{0.9}{
\begin{tabular}{lccc}
\toprule
Regularizer & Reward $\uparrow$ & Diversity $\uparrow$ & Coherence $\uparrow$ \\
\midrule
off-policy rollouts, off-policy Fisher $\bar{F}_t$                     & $\mathbf{9.26}$ & $0.723$ & $0.661$ \\
on-policy rollouts, off-policy Fisher $\bar{F}_t^u$    & $9.06$          & $0.721$ & $0.659$ \\
on-policy rollouts, on-policy Fisher $\tilde{F}_t^u$          & $8.92$          & $0.710$ & $0.656$ \\
KL                                            & $8.35$          & $0.725$ & $0.660$ \\
\bottomrule
\end{tabular}
}
\end{table}

%% file: tables/profile.tex
\begin{table}[t]
\centering
\small
\caption{%
Per-prompt runtime, generation throughput, and peak GPU memory for each
test-time alignment method on the SHP prompt set. Numbers are means over
10 timed prompts (after 5 warmup) on a single NVIDIA H200 NVL,
\texttt{max\_new\_tokens=512}, mean prompt length 131 tokens, no
cross-prompt batching. AISP uses $K{=}16$ Gaussian perturbations per
iteration over $N{=}16$ iterations; MISVO uses $K{=}16$
trajectories over $N{=}16$ gradient steps with frozen-Fisher
regularization; BoN draws $N{=}256$ top-$p$ samples.}
\label{tab:profile}
\begin{tabular}{lc rrr rrr}
\toprule
& & \multicolumn{3}{c}{LFM2.5-1.2B} & \multicolumn{3}{c}{Gemma3-4B-IT} \\
\cmidrule(lr){3-5} \cmidrule(lr){6-8}
Method & Hyperparams & Time (s) & Tok/s & Peak (GB) & Time (s) & Tok/s & Peak (GB) \\
\midrule
BoN   & $N{=}256$               & \textbf{24.4} & \textbf{5385} & 3.7  & \textbf{79.1}  & \textbf{1658} & 12.2 \\
AISP  & $K{=}16,\, N{=}16$ & 43.3          & 3053          & \textbf{3.5}          & 195.7          & 670           & \textbf{10.5} \\
MISVO & $K{=}16,\, N{=}16$      & 49.3          & 2666          & 10.7          & 167.1          & 784           & 38.8 \\
\bottomrule
\end{tabular}

\end{table}